\documentclass{article}

\usepackage{microtype}
\usepackage{graphicx}
\usepackage{subcaption}
\usepackage{booktabs} 

\usepackage{hyperref}

\usepackage[accepted]{icml2026}

\usepackage{amsmath}
\usepackage{amssymb}
\usepackage{mathtools}
\usepackage{amsthm}

\usepackage[capitalize,noabbrev]{cleveref}

\theoremstyle{plain}
\newtheorem{theorem}{Theorem}[section]
\newtheorem{proposition}[theorem]{Proposition}

\theoremstyle{definition}

\newtheorem{assumption}[theorem]{Assumption}
\theoremstyle{remark}

\icmltitlerunning{Neural Control: Adjoint Learning Through Equilibrium Constraints}

\begin{document}

\twocolumn[
  \icmltitle{Neural Control: Adjoint Learning Through Equilibrium Constraints}

  \begin{icmlauthorlist}
    \icmlauthor{Dezhong Tong}{1}
    \icmlauthor{Jiawen Wang}{2}
    \icmlauthor{Hengyi Zhou}{1}
    \icmlauthor{Yinglong Shen}{2}
    \icmlauthor{Xiaonan Huang}{1}
    \icmlauthor{M. Khalid Jawed}{2}
  \end{icmlauthorlist}

  \icmlaffiliation{1}{The Robotics Department, University of Michigan, Ann Arbor, US}
  \icmlaffiliation{2}{The Department of Mechanical and Aerospace Engineering, University of California, Los Angeles, US}

  \icmlcorrespondingauthor{Xiaonan Huang}{xiaonanh@umich.edu}
  \icmlcorrespondingauthor{M. Khalid Jawed}{khalidjm@seas.ucla.edu}

  \icmlkeywords{Implicit Differentiation, Differentiable Simulation, Model Predictive Control, Deformable Object Manipulation, Deep Equilibrium Models}

  \vskip 0.3in
]
\printAffiliationsAndNotice{}  

\begin{abstract}
Many physical AI tasks require sequential implicit computation: at each step, boundary controls are applied, and the resulting configuration is obtained by solving an equilibrium problem. 
This setting arises naturally in deformable object manipulation, where even bending a deformable linear object (DLO) to a target shape can be nonlinear and multistable: identical boundary conditions may produce different configurations depending on actuation history. 
Unlike explicit transition models, the control-to-configuration relation is implicit and history-dependent, making long-horizon learning and control brittle; backpropagating through iterative solves is also memory- and compute-intensive. 
We propose Neural Control, a boundary-control framework that propagates gradients through branch-dependent sequences of equilibrium solves rather than a single fixed point. 
Neural Control computes trajectory-dependent proxy gradients by differentiating equilibrium conditions with an adjoint formulation, avoiding solver unrolling while keeping forward rollouts on converged equilibria. 
Combined with receding-horizon continuation, Neural Control re-anchors optimization to realized equilibria and mitigates basin switching.
We validate Neural Control on simulated and real DLO manipulation, compare against SPSA and iCEM, and demonstrate applicability to a learned DEQ-style implicit equilibrium model.
\end{abstract}

\section{Introduction}
Many learning and control problems in physical AI require optimizing through sequential implicit computation: instead of producing the next state with an explicit transition map, each state is obtained by solving an implicit equilibrium or fixed-point problem. 
This setting arises naturally in physical AI, robotics, graphics, and simulation, where configurations are often defined as minimizers of an energy functional under boundary conditions, loads, or design parameters. 
In such systems, control changes the implicit problem itself: by adjusting boundary conditions or parameters, the agent shapes the energy landscape so that the resulting equilibrium satisfies a desired objective. 
The resulting control-to-configuration relation is therefore solver-defined rather than explicit, and it can be high-dimensional, nonlinear, and multistable.

This sequential implicit setting poses challenges beyond differentiating
a single fixed point. 
In multistable systems, the same control input can admit multiple valid equilibria, and the realized state is selected by the previous configuration, actuation history, and basin of attraction of the warm-started solver. 
Thus, the rollout is branch-dependent and history-dependent: small changes in actuation can move the system to a different equilibrium branch. 
This challenges the single-valued transition assumption underlying many model-learning and planning methods, and makes long-horizon credit assignment brittle~\cite{hausknecht2015deep, suh2022differentiable}.

A representative example is robotic manipulation of deformable
structures (Figure~\ref{fig:teaser}A), such as deformable linear objects
(DLOs)~\cite{mora2021pods, li2023dexdeform, tong2024sim2real}.
Even seemingly simple tasks, such as bending an elastic strip to a target shape, involve large geometrically nonlinear deformation, hysteresis, and multistable equilibria~\cite{qiao2020scalable, tong2025discrete, huang2025tutorial, bretl2014quasi, tong2021automated}. 
As illustrated in Figure~\ref{fig:teaser}B, identical boundary conditions can correspond to multiple stable configurations. 
DLO manipulation therefore provides a challenging testbed for sequential implicit optimization: each rollout step requires solving an equilibrium problem over a high-dimensional continuum discretization~\cite{yin2021modeling, moll2006path, bergou2008discrete}, and the realized branch depends on the actuation path.

Differentiable simulation provides gradients for physical optimization
problems~\cite{qiao2020scalable, mora2021pods, suh2022differentiable}, but long-horizon implicit rollouts introduce two additional difficulties.
First, multistable control requires branch-consistent,
trajectory-dependent sensitivities that follow the realized branch
selected by the solver~\cite{georgiev2024adaptive, gao2024adaptive}.
If the branch structure is ignored, learned dynamics may blur or average
over multiple outcomes; latent-state or multimodal models can partially
mitigate this issue~\cite{bishop1994mixture, hafner2019learning,
hafner2019dream}.
Second, naively backpropagating through every Newton or conjugate-gradient iteration is memory-intensive and can be numerically fragile near instabilities~\cite{gruslys2016memory, amos2017optnet, bai2019deep}.

Existing implicit-layer methods typically differentiate a single
solver-defined object, such as one optimization layer or KKT
system~\cite{amos2017optnet, agrawal2019differentiable,
amos2018differentiable}, or one fixed point in a deep equilibrium
model~\cite{bai2019deep}.
In contrast, our setting requires propagating useful gradients through a
sequence of warm-started, branch-dependent equilibrium solves.
The central challenge is therefore not only to differentiate an implicit state, but to obtain trajectory-level sensitivities that remain aligned with the realized equilibrium branch.

We propose Neural Control\footnote{\label{fn:projectpage}Code, videos, and additional results are available at
\url{https://structurescomp.github.io/neural-control/}.}, a boundary-control framework for
long-horizon optimization through sequential implicit equilibria.
The forward rollout is solver-consistent: each state is obtained by solving the equilibrium equations to convergence along the realized branch.
Inspired by adjoint sensitivity methods in continuous-depth learning~\cite{chen2018neural}, we differentiate the equilibrium conditions via an adjoint formulation to compute trajectory-dependent proxy gradients, avoiding unrolling of the inner Newton/CG iterations.
This provides a memory-efficient gradient pathway for optimizing control
signals through sequences of implicit solves.

To improve robustness in multistable regimes, we combine these gradients with receding-horizon continuation (RHC), related in spirit to
receding-horizon model predictive control~\cite{chua2018deep}.
In this setting, RHC serves as a continuation and branch-tracking mechanism: by repeatedly re-anchoring optimization to realized equilibria, it localizes proxy-gradient error, reduces basin drift, and mitigates unintended switches between equilibrium branches.

We validate Neural Control on shape and trajectory control of DLOs using
both simulation and real-robot experiments under noise, model mismatch,
and gravity. 
We compare against black-box gradient-free baselines, including SPSA~\cite{spall2002multivariate} and iCEM~\cite{pinneri2021sample}, with CEM~\cite{botev2013cross} reported in the appendix as an additional classical baseline, and show that Neural Control achieves substantially lower loss with significantly less wall-clock time.
To further demonstrate that the framework can operate beyond analytical mechanics models, we also evaluate it on a learned DEQ-style implicit equilibrium model trained from experimental force--strain data.

This paper makes the following contributions:
(i) a sequential implicit optimization formulation for branch-dependent
equilibrium rollouts, clarifying the distinction from single-step
implicit layers and fixed-point models;
(ii) a proxy-adjoint gradient method that computes trajectory-dependent
sensitivities through converged equilibrium solves without unrolling
inner Newton/CG iterations;
(iii) a receding-horizon continuation scheme that re-anchors optimization
to realized equilibria and improves branch consistency in multistable
regimes;
(iv) empirical validation on simulated and real DLO manipulation,
comparison with modern gradient-free baselines, including iCEM, and an
additional learned DEQ-style implicit-model experiment.

\begin{figure}
    \centering
    \includegraphics[width=\linewidth]{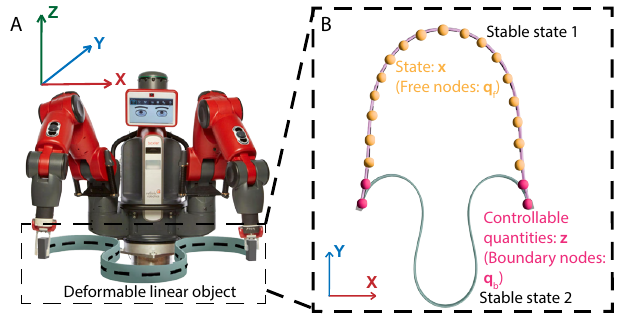}
    \caption{(A) Robotic manipulation of a deformable linear object (e.g., elastic strips). (B) Discretization into boundary and free nodes. Varying the controllable quantities $\mathbf z$ shapes the resulting state $\mathbf x$. Multiple stable shapes may exist under the same $\mathbf z$.}
    \label{fig:teaser}
\end{figure}

\section{Related Work}
\paragraph{Optimization and Learning for Long-Horizon, Path-Dependent Control.}
Long-horizon control and sequential decision-making are central themes in
machine learning, where performance is often limited by delayed credit
assignment, partial observability, and, in model-based settings,
compounding model error.
A common response is to treat history dependence as partial observability and augment policies or value functions with memory, e.g., recurrent reinforcement learning (RL)~\cite{hausknecht2015deep}, or to learn latent state-space models that support planning through imagined rollouts~\cite{hafner2019learning,hafner2019dream}.
Another dominant paradigm is model-based RL with receding-horizon planning:
actions are optimized over a short horizon and repeatedly replanned using
model predictive control (MPC), improving robustness to model error and
regime changes~\cite{chua2018deep}.

Black-box and sampling-based optimizers are also widely used for
long-horizon problems because they avoid differentiating through complex
transitions and can explore multimodal outcomes.
Representative methods include simultaneous perturbation stochastic approximation (SPSA)~\cite{spall2002multivariate}, the cross-entropy method
(CEM)~\cite{botev2013cross}, sample-efficient variants such as iCEM~\cite{pinneri2021sample}, and parameter-space search methods such as
evolution strategies (ES)~\cite{salimans2017evolution} and augmented
random search (ARS)~\cite{mania2018simple}.
Related work also makes CEM differentiable (DCEM)~\cite{amos2020differentiable}. 
However, these methods typically require many full rollouts per update, and their cost grows quickly with horizon length, action dimension, and population size.
This cost becomes especially acute in our setting, where each rollout step itself requires an expensive implicit equilibrium solve. 
Our method retains the stabilizing effect of receding-horizon replanning, but computes trajectory-level sensitivities through the realized implicit
rollout, enabling efficient gradient-based optimization.

\paragraph{Memory-Efficient Gradients Through Implicit Computation.}
Backpropagation through time (BPTT) provides a direct way to obtain
trajectory-dependent gradients by propagating sensitivities across
timesteps~\cite{williams1990efficient}, but storing intermediate states
makes memory scale with horizon length, motivating checkpointing and
recomputation strategies~\cite{gruslys2016memory}. 
In differentiable physics and control, this issue is amplified when each transition is computed by an iterative implicit solver, such as Newton or conjugate gradient, since differentiating through all solver iterations is
memory-intensive and can be numerically brittle near instabilities~\cite{amos2017optnet,suh2022differentiable}.

A complementary line of work avoids solver unrolling by differentiating
through optimality, fixed-point, or KKT conditions.
Examples include differentiable optimization layers such as OptNet~\cite{amos2017optnet} and differentiable convex optimization layers~\cite{agrawal2019differentiable}, deep equilibrium models (DEQs)~\cite{bai2019deep}, differentiable MPC via KKT conditions~\cite{amos2018differentiable}, and Pontryagin differentiable
programming (PDP)~\cite{jin2020pontryagin}. 
Inspired by adjoint sensitivity methods for continuous-depth models~\cite{chen2018neural}, these methods show that gradients can often be obtained without unrolling the full computational procedure.

Our setting differs from these methods in that the object being
differentiated is not a single optimization layer, fixed point, or KKT
system. 
Instead, we require gradients through a sequence of warm-started equilibrium solves, where multistability makes the realized
trajectory branch-dependent. 
Compared to BPTT through solver iterations, Neural Control avoids unrolling inner solves; compared to standard implicit layers, it propagates proxy-adjoint sensitivities through long-horizon, branch-dependent implicit rollouts.

\paragraph{Differentiable Models for Physical Learning and Control.}
Differentiable modeling and simulation have become core primitives for
gradient-based learning, planning, and inverse problems. 
ML toolchains such as DiffTaichi~\cite{hu2019difftaichi} and Brax~\cite{freeman2021brax} make optimization through physics practical at scale. 
Differentiable simulators have also been extended from rigid bodies to deformable and soft systems, including ChainQueen~\cite{hu2019chainqueen}, DiffPD~\cite{du2021diffpd}, and differentiable soft multi-body simulation~\cite{qiao2021differentiable}. 
For flexible structures, discrete differential geometry (DDG) discretizations provide geometry-preserving, energy-based models with analytic derivatives~\cite{bergou2008discrete,grinspun2003discrete,huang2025tutorial, tong2025discrete}.
Benchmarks such as SoftGym~\cite{lin2021softgym} and PlasticineLab~\cite{huang2021plasticinelab} further highlight the difficulty of long-horizon deformable manipulation.

Learned differentiable simulators and graph-based dynamics models can
provide fast predictors for physical systems~\cite{sanchez2020gns,
pfaff2021meshgraphnets}, while analytic gradients from differentiable multiphysics simulations have also been integrated into RL updates, e.g.,
SAPO~\cite{xing2025sapo}. 
However, near multistability and bifurcation boundaries, deterministic predictors can blur branch selection unless multimodality or latent modes are modeled explicitly~\cite{bishop1994mixture}.
Our work is complementary to these modeling efforts: rather than
proposing a new simulator, we focus on the gradient pathway for
optimizing through sequential implicit models. This perspective applies
both to analytical mechanics models and to learned implicit equilibrium
models, as demonstrated by our DEQ-style validation.

\section{Methodology}
\label{sec:method}
We study learning and control problems where the system state is defined implicitly by an equilibrium condition rather than an explicit dynamics model. A continuation parameter $\lambda\in[0,1]$ indexes the process, with $\lambda=0$ the initial configuration and $\lambda=1$ the terminal configuration. At each $\lambda$, the equilibrium state $\mathbf x(\lambda)$ is obtained by solving
\begin{equation}
\mathcal G(\mathbf x(\lambda),\mathbf z(\lambda))=\mathbf 0,
\label{eq:equilibrium}
\end{equation}
where $\mathbf z(\lambda)$ denotes controllable quantities entering the
equilibrium problem, such as boundary conditions, loads, or design
parameters. Thus, the state is not produced by an explicit transition
map of the form $\mathbf x_{k+1}=\mathcal F(\mathbf x_k,\mathbf z_k)$,
but by a solver-selected equilibrium. 
In mechanics, Eq.~\eqref{eq:equilibrium} represents quasi-static force balance: inertia and vibrational effects are neglected, so internal conservative forces balance applied loads. A common instance is
\begin{equation}
\mathcal G(\mathbf x,\mathbf z)=\nabla_{\mathbf x}E(\mathbf x,\mathbf z)-\mathbf F^{\mathrm{ext}},
\label{eq:equilibrium_example}
\end{equation}
where $E$ is total potential energy and $\mathbf F^{\mathrm{ext}}$ collects external loads.

Equilibrium-governed systems pose two challenges: (i) $\mathbf x(\lambda)$ is typically computed by iterative solvers (e.g., Newton--CG), so differentiating through solver iterations across long horizons is expensive; and (ii) equilibria can be multistable, so for fixed $\mathbf z$, multiple stable solutions may exist and warm-started solvers follow a realized branch through their basins of attraction. 
We therefore generate rollouts by warm-started equilibrium solves along the continuation path. 
In all experiments, each forward state is solved to the prescribed equilibrium tolerance, so the backward approximation does not arise from truncating or unrolling the forward solver.

A neural controller parameterized by $\Theta$ produces a continuation control signal
\begin{equation}
\mathbf u(\lambda)=\mathbf u_\Theta(\lambda),
\label{eq:u_theta}
\end{equation}
which drives $\mathbf z(\lambda)$ via known continuation dynamics
\begin{equation}
\frac{d\mathbf z}{d\lambda}=f(\mathbf z(\lambda),\mathbf u(\lambda)),
\qquad \mathbf z(0)=\mathbf z_0.
\label{eq:z_dyn}
\end{equation}
Given $\mathbf z(\lambda)$, the realized equilibrium $\mathbf x(\lambda)$ is obtained by repeatedly solving Eq.~\eqref{eq:equilibrium} with warm starts from the previous solution, yielding a branch-consistent rollout.

\subsection{Adjoint Gradients via Proxy Sensitivity Dynamics}
\label{sec:proxy_adjoints}
At any converged equilibrium $(\mathbf x(\lambda),\mathbf z(\lambda))$
satisfying $\mathcal G(\mathbf x,\mathbf z)=\mathbf 0$, and assuming the
realized branch is locally smooth with nonsingular $\mathbf G_x$, implicit differentiation gives the local equilibrium sensitivity
\begin{equation}
\mathbf S(\lambda):=\frac{\partial \mathbf x}{\partial \mathbf z}(\lambda),
\qquad
\mathbf G_x(\lambda)\,\mathbf S(\lambda)=-\mathbf G_z(\lambda),
\label{eq:S_cont}
\end{equation}
where $\mathbf G_x=\partial\mathcal G/\partial\mathbf x$ and
$\mathbf G_z=\partial\mathcal G/\partial\mathbf z$.
Near bifurcations or ill-conditioned equilibria, $\mathbf G_x$ may become singular or poorly conditioned, and the local sensitivity can become unreliable.
We never form $\mathbf S(\lambda)$ explicitly. To compute a product
$\mathbf S(\lambda)^\top\mathbf v$, we solve the adjoint linear system
\begin{equation}
\mathbf G_x(\lambda)^\top \boldsymbol\eta = \mathbf v
\end{equation}
and return
\begin{equation}
\mathbf S(\lambda)^\top\mathbf v
=
-\mathbf G_z(\lambda)^\top\boldsymbol\eta .
\end{equation}
This requires one adjoint linear solve and avoids constructing the dense
sensitivity matrix.

Along a smooth realized branch, the total derivative satisfies
\begin{equation}
\frac{d\mathbf x}{d\lambda}
=
\mathbf S(\lambda)\frac{d\mathbf z}{d\lambda}
+
\frac{\partial \mathbf x}{\partial \lambda}.
\label{eq:dxdl_general}
\end{equation}
In our instantiation, $\mathcal G$ has no explicit $\lambda$-dependence: any scheduling is absorbed into $\mathbf z(\lambda)$.
Consequently, $\partial \mathbf x/\partial \lambda=\mathbf 0$ and
\begin{equation}
\frac{d\mathbf x}{d\lambda}
=
\mathbf S(\lambda)\frac{d\mathbf z}{d\lambda}.
\label{eq:dxdl_exact}
\end{equation}
Combining Eq.~\eqref{eq:dxdl_exact} with the continuation dynamics Eq.~\eqref{eq:z_dyn} gives the continuous sensitivity identity
\begin{equation}
\frac{d\mathbf x}{d\lambda}
=
\mathbf S(\lambda)\,f(\mathbf z(\lambda),\mathbf u(\lambda)).
\label{eq:dxdl_proxy}
\end{equation}

Eq.~\eqref{eq:dxdl_proxy} defines the proxy sensitivity dynamics
used for backward sensitivity propagation. The forward states
$\mathbf x(\lambda)$ are still obtained by solving
Eq.~\eqref{eq:equilibrium} to convergence at each continuation step.
Thus, the forward rollout remains solver-consistent, while the backward
pass uses local implicit sensitivities evaluated along the realized
branch.

The main analytic approximation in our gradient construction is a frozen-tangent discretization. We evaluate $\mathbf S(\lambda)$ on the realized rollout at discrete continuation steps and treat it as locally constant during the backward pass, thereby ignoring derivatives of $\mathbf S$ with respect to $\mathbf x$, $\mathbf z$, and $\lambda$ along the segment. 
This avoids higher-order derivatives of Eq.~\eqref{eq:equilibrium}. 
The approximation is local: it assumes that the realized equilibrium branch is smooth and that $\mathbf G_x$ is nonsingular and varies moderately over each continuation step. 
Near bifurcations, ill-conditioned equilibria, or high-curvature branches, the local tangent can change rapidly and the proxy gradient may become less accurate. 
Appendix~\ref{app:frozen_tangent} formalizes this approximation and shows that, under these local regularity assumptions, the resulting proxy gradient is first-order accurate.

\paragraph{Terminal Objective.}
We first consider a terminal loss
\begin{equation}
\mathcal L=\phi(\mathbf x(1),\mathbf z(1)).
\label{eq:terminal_loss_cont}
\end{equation}
Define proxy-system costates $\mathbf a(\lambda):=\partial \mathcal L/\partial\mathbf x(\lambda)$ and
$\mathbf g(\lambda):=\partial \mathcal L/\partial\mathbf z(\lambda)$.
Under Eq.~\eqref{eq:dxdl_proxy} and frozen-tangent approximation, the terminal conditions are
\begin{equation}
\begin{aligned}
\mathbf a(1)
&=
\nabla_{\mathbf x}\phi(\mathbf x(1),\mathbf z(1)),\\
\mathbf g(1)
&=
\nabla_{\mathbf z}\phi(\mathbf x(1),\mathbf z(1)).
\end{aligned}
\label{eq:terminal_bc_cont}
\end{equation}

Let $\mathbf A(\lambda):=\partial f/\partial\mathbf z(\mathbf z(\lambda),\mathbf u(\lambda))$ and $\mathbf B(\lambda):=\partial f/\partial\mathbf u(\mathbf z(\lambda),\mathbf u(\lambda))$.
Under frozen tangent, we neglect derivatives of $\mathbf S(\lambda)$ along the branch. The costates satisfy the backward ODEs, integrated from $\lambda=1$ to $0$,
\begin{equation}
\begin{aligned}
-\frac{d\mathbf a}{d\lambda} &= \mathbf 0,\\
-\frac{d\mathbf g}{d\lambda}
& =
\mathbf A(\lambda)^\top\!\Big(\mathbf g(\lambda)+\mathbf S(\lambda)^\top \mathbf a(\lambda)\Big),
\end{aligned}
\label{eq:terminal_adjoint_odes}
\end{equation}
and the instantaneous control gradient is
\begin{equation}
\frac{d\mathcal L}{d\mathbf u(\lambda)}
=
\mathbf B(\lambda)^\top\!\Big(\mathbf g(\lambda)+\mathbf S(\lambda)^\top \mathbf a(\lambda)\Big).
\label{eq:dL_du_cont}
\end{equation}
The parameter gradient follows by the chain rule
\begin{equation}
\frac{d\mathcal L}{d\Theta}
=
\int_{0}^{1}
\left(\frac{\partial \mathbf u_\Theta(\lambda)}{\partial \Theta}\right)^\top
\frac{d\mathcal L}{d\mathbf u(\lambda)}\,d\lambda .
\label{eq:dL_dTheta_cont}
\end{equation}

\subsection{Adjoints for Trajectory Objectives}
\label{sec:adjoint_trajectory}
Many equilibrium-control tasks specify objectives along the realized branch rather than only at $\lambda=1$.
We therefore consider
\begin{equation}
\mathcal L
=
\phi(\mathbf x(1),\mathbf z(1))
+
\int_{0}^{1} w(\lambda)\,\ell(\mathbf x(\lambda),\mathbf z(\lambda))\,d\lambda,
\label{eq:traj_loss_cont}
\end{equation}
where $w(\lambda)\ge 0$ is user-specified. The terminal conditions remain
Eq.~\eqref{eq:terminal_bc_cont}. Under Eq.~\eqref{eq:dxdl_proxy} and
frozen tangent, the backward ODEs become
\begin{equation}
\begin{aligned}
-\frac{d\mathbf a}{d\lambda}
= &
w(\lambda)
\nabla_{\mathbf x}
\ell(\mathbf x(\lambda),\mathbf z(\lambda)), \\
-\frac{d\mathbf g}{d\lambda}
=
&
w(\lambda)
\nabla_{\mathbf z}
\ell(\mathbf x(\lambda),\mathbf z(\lambda))
\\
&+
\mathbf A(\lambda)^\top
\Big(
\mathbf g(\lambda)
+
\mathbf S(\lambda)^\top\mathbf a(\lambda)
\Big).
\end{aligned}
\label{eq:traj_a_g_ode}
\end{equation}
The control and parameter gradients retain the same form as
Eqs.~\eqref{eq:dL_du_cont}--\eqref{eq:dL_dTheta_cont}.

\subsection{Receding-Horizon Continuation Control}
\label{sec:receding_horizon_control}
Over long continuations, proxy-gradient and discretization errors can
accumulate.
In multistable regions, these errors may push the warm-started equilibrium solve across basins of attraction, causing the rollout to switch to an unintended equilibrium branch. 
We therefore solve the task as a sequence of short continuation segments, in the spirit of receding-horizon optimization. 
In our setting, receding-horizon continuation (RHC) serves as a continuation and branch-tracking mechanism: it repeatedly re-anchors optimization to the realized equilibrium branch.

At segment index $m$, starting from $(\mathbf x^{(m)}_0,\mathbf z^{(m)}_0)$, we optimize a control schedule over a local homotopy $\lambda\in[0,1]$, implemented with $H$ continuation steps of size $1/H$. 
The continuation control is generated by a lightweight neural controller
\begin{equation}
\mathbf u(\lambda)=\mathbf u_{\Theta^{(m)}}(\lambda),
\label{eq:rhc_policy}
\end{equation}
with fixed architecture and parameters $\Theta^{(m)}$ reinitialized at the start of each segment and optimized within the segment.
We minimize a short-horizon segment objective $\mathcal L^{(m)}$ of the same form as Eq.~\eqref{eq:traj_loss_cont}, restricted to the segment rollout:
\begin{equation}
\begin{aligned}
\min_{\Theta^{(m)}}\; \mathcal L^{(m)} = &\phi(\mathbf x^{(m)}(1),\mathbf z^{(m)}(1)) \\
+ & \int_{0}^{1}w(\lambda)\,
\ell(\mathbf x^{(m)}(\lambda),\mathbf z^{(m)}(\lambda))\,d\lambda.
\end{aligned}
\label{eq:rhc_theta_obj}
\end{equation}
subject to the equilibrium constraint~\eqref{eq:equilibrium}, continuation dynamics~\eqref{eq:z_dyn}, and bounded progress (e.g., $\|\mathbf u(\lambda)\|\le u_{\max}$). 

We optimize $\Theta^{(m)}$ using the proxy-adjoint gradients above and stop early when the segment objective stalls, or after a fixed budget.
Executing the optimized segment yields
$(\mathbf x^{(m)}(1),\mathbf z^{(m)}(1))$, and we shift the horizon by
\begin{equation}
\mathbf x^{(m+1)}_0 \leftarrow \mathbf x^{(m)}(1),
\qquad
\mathbf z^{(m+1)}_0 \leftarrow \mathbf z^{(m)}(1).
\label{eq:rhc_shift_cont}
\end{equation}
This segment-wise design localizes proxy-gradient error, reduces basin
drift, and improves branch consistency in multistable landscapes.
We provide the discrete derivation, implementation details, and pseudocode
in Appendix~\ref{app:discrete_proxy_adjoint} and
Algorithm~\ref{alg:adjoint_rhc}.

\section{Results}
\label{sec:results}
We evaluate Neural Control on elastic-strip manipulation tasks
(Figure~\ref{fig:teaser}A).
The strip undergoes large deformations and admits multistable equilibria, inducing strong path dependence: small continuation errors can accumulate and trigger undesired switches between equilibrium branches. 
We consider three tasks that progressively stress these effects: (i) point-to-point manipulation with a terminal objective, 
(ii) path following of a designated strip point with a trajectory objective, and 
(iii) shape formation in multistable regimes requiring
branch selection.

All methods share the same forward equilibrium solver, which computes strip configurations by solving implicit equilibria to convergence at each continuation step.
We compare our full method (Adjoint+RHC) against gradient-free baselines, including SPSA and iCEM, and ablate our approach by removing RHC (Adjoint-only). 
CEM results are provided in Appendix~\ref{app:cem_results} as an additional classical gradient-free baseline.

Each method is run under a fixed update budget. 
We report both best achieved loss and total wall-clock time, including all equilibrium solves and additional tangent/adjoint linear-solve overheads. 
Finally, we validate the approach in real-world robotic experiments; an additional learned DEQ-style implicit-model experiment is provided in Appendix~\ref{app:deq_validation}.

\subsection{Experimental Setup}
\label{sec:exp_setup}

We model the manipulated object as a planar elastic strip with stretching and bending energy.
At each continuation step, the forward rollout computes the realized configuration by solving the corresponding
equilibrium equations to convergence under imposed boundary constraints.
We partition degrees of freedom into boundary nodes $\mathbf q_b$ and free nodes $\mathbf q_f$ and set
$\mathbf z\equiv\mathbf q_b$, $\mathbf x\equiv\mathbf q_f$ (Figure~\ref{fig:teaser}B): control acts only through boundary conditions, while free nodes are solved implicitly and used to evaluate task losses.
The strip is held by clamped boundary conditions at both ends; Tasks~1--2 actuate transverse (global x-axis in our setup) displacements at both ends, and Task~3 actuates $(x,y)$ displacement and in-plane rotation at one end, with the other end passively clamped.
Unless otherwise stated, we discretize the strip with $N + 1=101$ nodes and ignore gravity in numerical experiments.
The stretching and bending stiffness are fixed to
$k_s = 0.314~\mathrm{N}$ and $k_b = 7.85\times 10^{-8}~\mathrm{N\,m^2}$.
Full discretization and solver details are provided in the Appendix~\ref{app:strip_solver}.

\subsection{Tasks and Metrics}
\label{subsec:tasks_metrics}

We evaluate three numerical tasks that progressively test terminal objectives, trajectory objectives, and multistable shape/branch selection. 
In all tasks, boundary degrees of freedom define the controllable variable $\mathbf z\equiv\mathbf q_b$, and the free degrees of freedom $\mathbf x\equiv\mathbf q_f$ are solved implicitly from equilibrium and
used to evaluate losses.
Tasks~1--2 start from the same pre-deformed configuration
(Figure~\ref{fig:task1_2}A), while Task~3 starts from a straight strip (Figure~\ref{fig:task3}).

The continuation control is produced by a neural controller
$\mathbf u(\lambda)=\mathbf u_\Theta(\lambda)$, where
$\mathbf u(\lambda)$ specifies the rate of boundary actuation along continuation.
In Tasks~1--2, $\mathbf u(\lambda)\in\mathbb R^2$ outputs the transverse ($x$-direction) displacement rates of the left and right clamped ends.
In Task~3, $\mathbf u(\lambda)\in\mathbb R^3$ outputs the $(x,y)$ displacement rates and in-plane rotation rate of the manipulated end, while the other end remains passively clamped.

\textbf{Task 1: Point-to-Point Node Targeting.}
We select a strip node $i$ and a target position $\mathbf p_i^\star$ within the reachable workspace.
The goal is to drive the node to the target at the end of continuation. 
We use the terminal loss
\begin{equation}
\mathcal L_{\text{pos}}=\frac{1}{2}\big\|\mathbf p_i(1)-\mathbf p_i^\star\big\|_2^2,
\label{eq:task1_loss}
\end{equation}
and declare success if $\mathcal L_{\text{pos}}\le \varepsilon_{\text{pos}}$.

\textbf{Task 2: Trajectory Tracking.}
We track a prescribed planar path with a designated node, chosen as the middle node.
Let $\mathbf p^\star_{\mathrm{mid}}(\lambda)$ denote the reference trajectory, e.g., sinusoidal, circular, square-wave, or triangular paths (Figure~\ref{fig:task1_2}B).
We define
\begin{equation}
\ell_{\text{traj}}(\lambda)=\frac{1}{2}\big\|\mathbf p_{\mathrm{mid}}(\lambda)-\mathbf p_{\mathrm{mid}}^\star(\lambda)\big\|_2^2,
\label{eq:task2_running}
\end{equation}
and evaluate the trajectory objective
\begin{equation}
\mathcal L_{\text{traj}}=\int_0^1 \ell_{\text{traj}}(\lambda)\,d\lambda,
\label{eq:task2_obj}
\end{equation}
implemented as a Riemann sum over realized continuation steps and, for RHC, over the concatenated executed trajectory. 
We declare success if $\mathcal L_{\text{traj}}\le\varepsilon_{\text{traj}}$.

\textbf{Task 3: Shape Formation Under Multistability.}
We shape the strip by manipulating one clamped end while the other end remains passively clamped.
For fixed boundary conditions, multiple stable equilibria may coexist; given a target shape, the goal is to reach the corresponding equilibrium branch. 
We measure shape discrepancy by curvature matching:
\begin{equation}
\mathcal L_{\kappa}=\int_{0}^{L}\big(\kappa(s;1)-\kappa^\star(s)\big)^2\,ds,
\label{eq:task3_loss}
\end{equation}
where $\kappa^\star(s)$ is the target curvature. We declare success if $\mathcal L_{\kappa}\le \varepsilon_{\kappa}$.

\subsection{Main Numerical Results}
\label{subsec:main_results}

\begin{figure*}
    \centering
    \includegraphics[width=\linewidth]{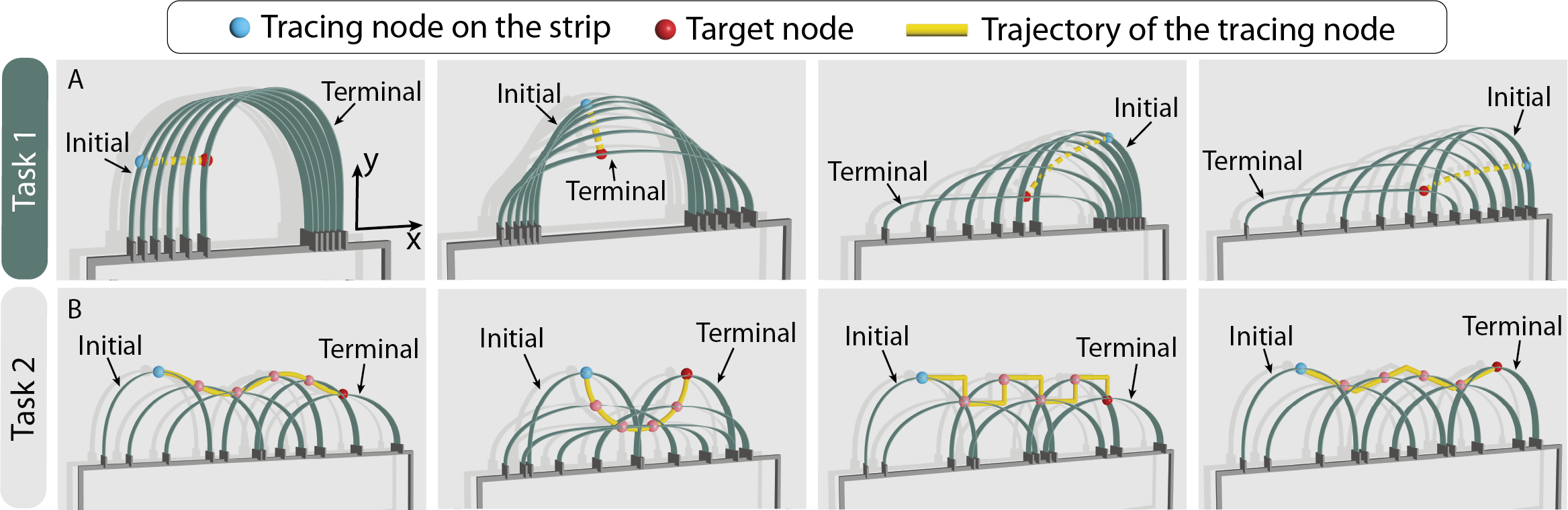}
    \caption{Numerical results for Tasks~1--2. (A) Four representative cases of two-end boundary actuation that drive a selected strip node to a user-specified target position. (B) Four representative cases of two-end boundary actuation that guide the strip midpoint along a prescribed planar trajectory. (Shadow is rendered for visualization only.)}
    \label{fig:task1_2}
\end{figure*}

\begin{figure}[b!]
    \centering
    \includegraphics[width=\linewidth]{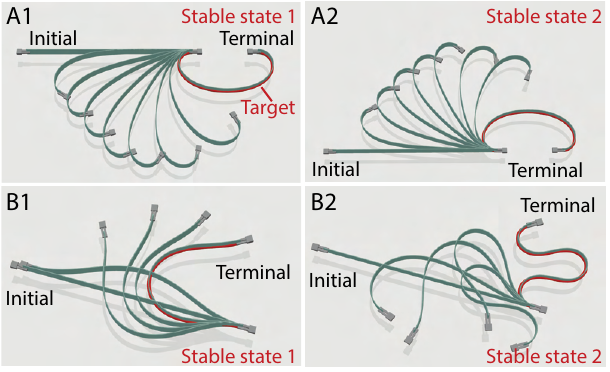}
    \caption{Neural continuation control for multistable shape formation (Task~3).
    (A1--A2) Two mirrored C-shaped stable equilibria arise under the same boundary condition (identical end pose), illustrating multistability.
    (B1--B2) A U-shaped and an M-like stable equilibrium are also obtained under the same boundary conditions.}
    \label{fig:task3}
\end{figure}

We report the main numerical results of Neural Control on the three tasks in Section~\ref{subsec:tasks_metrics}.
Unless otherwise noted, all experiments use the same strip discretization and physical parameters as in Section~\ref{sec:exp_setup}, with solver and hyperparameter details deferred to the Appendix~\ref{app:hyperparams}.

\textbf{Tasks~1--2: Two-End Boundary Actuation.}
Figure~\ref{fig:task1_2} summarizes four representative instances of point-to-point node targeting (Task~1; Figure~\ref{fig:task1_2}A) and midpoint trajectory tracking (Task~2; Figure~\ref{fig:task1_2}B).
In Task~1, the controller steers a user-selected node to the target position by coordinating the transverse actuation of both clamped ends, producing coupled rigid motion and elastic deformation along the realized equilibrium branch. 
In Task~2, the controller guides the strip midpoint to follow prescribed nontrivial planar trajectories; the tracked paths closely match the references across all cases, confirming that the method can optimize objectives that depend on intermediate equilibria rather than only the terminal state.

\textbf{Task~3: Multistable Shape Formation.}
Figure~\ref{fig:task3} visualizes shape formation under multistability when manipulating a single clamped end while the other end remains passively clamped. 
For the same boundary condition, multiple stable equilibria coexist, including mirror-symmetric C-shaped solutions (Figure~\ref{fig:task3}A1--A2) and distinct U/M-like shapes (Figure~\ref{fig:task3}B1--B2).
Despite these competing basins, curvature-based shaping drives the strip to the specified target configuration, indicating robust branch selection in regimes where small continuation errors can otherwise induce unintended basin switches.

\subsection{Baselines and Ablations}
\label{subsec:baselines_ablations}

We compare four optimizers for elastic-strip manipulation: (i) SPSA, (ii) iCEM, (iii) our neural continuation control without RHC (Adjoint-only; full-horizon), and (iv) our full method, which combines the same adjoint-based neural control with RHC (Adjoint+RHC). 
All methods share the same forward simulator (implicit equilibrium solves to convergence at each continuation step), and the same controller parameterization $\mathbf u_\Theta(\lambda)$ with $\lambda$ as the sole input.
They differ only in how the controller parameters $\Theta$ are updated and whether receding-horizon execution is used.
CEM results are reported in Appendix~\ref{app:cem_results} as an additional classical gradient-free baseline.

\begin{figure*}
    \centering
    \includegraphics[width=\linewidth]{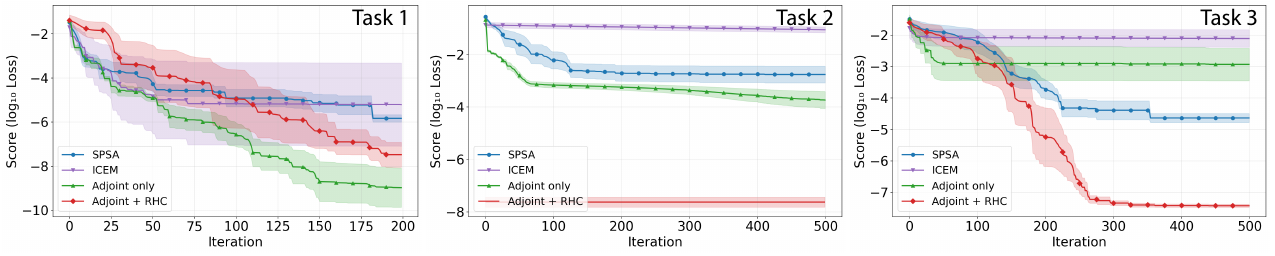}
    \caption{
    Optimization steps (iterations) vs.\ score (log best-so-far loss) for Tasks~1--3 (mean $\pm$1 stderr over four cases). For Task~2, Adjoint+RHC executes the trajectory in segments; we therefore report its full-horizon trajectory objective $\mathcal L_{\mathrm{traj}}$ only when a complete rollout is available (i.e., at the end of the executed horizon), and use those values to form the best-so-far curve.}
    \label{fig:loss}
\end{figure*}

\textbf{Budgeting and Evaluation.}
Each method is run under a fixed update budget.
An objective evaluation corresponds to one full rollout that computes the task loss from implicit equilibrium solves.
The number of objective evaluations per update is method-dependent: SPSA uses two evaluations per update, iCEM evaluates a population of sampled candidates, and Adjoint-only / Adjoint+RHC use one rollout per update. 
Adjoint-based methods additionally solve tangent/adjoint linear systems to obtain gradients via implicit sensitivities, without backpropagating through the equilibrium solver iterations. 
For RHC, an update refers to one optimizer update within the current segment, and all reported losses are computed on the concatenated executed trajectory across segments.
We report the best loss achieved within the update budget and the total wall-clock time in Table~\ref{tab:theory_empirical}; wall-clock time includes all equilibrium solves and additional linear-solve overheads.

\textbf{Baselines and Ablation.}
SPSA and iCEM are used as black-box optimizers over the controller parameters $\Theta$ and are run without RHC.
SPSA forms a stochastic descent direction using two-sided random perturbations of $\Theta$.
iCEM samples a population of candidate controller parameters, selects elite candidates, smooths the population update, and reinserts the current best candidate.
In our implementation, we use population size $P=10$, elite fraction $0.3$, smoothing coefficient $\alpha=0.25$, minimum standard deviation $10^{-2}$, elite reuse fraction $1.0$, no population decay, and best-candidate reinsertion.
Additional hyperparameters are provided in Appendix~\ref{app:hyperparams}.
Our ablation isolates the contribution of RHC by comparing Adjoint-only against Adjoint+RHC.
In Adjoint+RHC, we split the continuation into short segments of $H$ steps and solve a sequence of segment subproblems; the controller parameters $\Theta$ are reinitialized at the start of each segment.
All reported task losses are evaluated using the task definitions in Section~\ref{subsec:tasks_metrics} on the concatenated executed rollout across segments.   

\textbf{Convergence and Robustness.}
Figure~\ref{fig:loss} plots the log best-so-far objective value versus optimizer updates for all methods and tasks.
Task~1 (node targeting) is the simplest setting: all methods reduce the loss early, but the gradient-free baselines plateau under the fixed update budget.
Both adjoint-based variants continue improving and reach lower final loss, with Adjoint-only converging fastest in this short-horizon regime where proxy-adjoint gradients remain accurate.

Task~2 (trajectory tracking) is substantially harder because the objective depends on intermediate equilibria along the realized branch.
SPSA and iCEM again stall under the fixed update budget.
Adjoint-only succeeds on smooth references but struggles on non-smooth targets, where long-horizon accumulation makes optimization brittle.
In contrast, Adjoint+RHC achieves consistent convergence across all trajectory types, indicating that segment-wise re-anchoring to realized equilibria is critical for stable trajectory-level control.

Task~3 (multistable shape formation) is the most challenging regime: multiple stable equilibria coexist under identical boundary conditions, so small errors can induce undesired basin switches.
SPSA and iCEM stagnate early, and Adjoint-only, despite a fast initial decrease, often becomes trapped in an incorrect branch and fails the curvature tolerance.
Adjoint+RHC consistently reaches the lowest curvature losses and reliably selects the specified equilibrium branch, demonstrating the importance of combining adjoint-based gradients with RHC in multistable landscapes.

\textbf{Efficiency Summary.}
Table~\ref{tab:theory_empirical} reports best loss and wall-clock time (mean $\pm$ std) under the fixed update budget.
Adjoint+RHC is consistently the most efficient method, achieving substantially lower loss with lower runtime than SPSA and iCEM.
This gap is expected: gradient-free methods spend most compute on repeated full equilibrium-defined rollouts, whereas adjoint-based updates extract a search direction from one realized rollout plus additional tangent/adjoint linear solves.
Compared to Adjoint-only, RHC mitigates long-horizon error accumulation and reduces trajectory-storage memory by backpropagating only through short segments, yielding the best time--accuracy trade-off in our experiments.

\begin{table*}[t]
\centering
\scriptsize
\setlength{\tabcolsep}{4pt}
\resizebox{\textwidth}{!}{%
\begin{tabular}{lcc|cc|cc|cc}
\hline
& \multicolumn{2}{c|}{Theory (per optimizer update)}
& \multicolumn{2}{c|}{Task 1 (200 updates)}
& \multicolumn{2}{c|}{Task 2 (500 updates)}
& \multicolumn{2}{c}{Task 3 (500 updates)} \\
Method
& Time/update
& Memory/update
& Time (s)$\downarrow$
& Best loss$\downarrow$
& Time (s)$\downarrow$
& Best loss$\downarrow$
& Time (s)$\downarrow$
& Best loss$\downarrow$ \\
\hline
SPSA
& $\mathcal O(2K\,C_{\mathrm{eq}})$
& $\mathcal O(n_\Theta)$
& $727.3\pm108.8$ & $2.0{\times}10^{-6}\pm1.8{\times}10^{-6}$
& $3612.9\pm1231.9$ & $5.4{\times}10^{-3}\pm7.7{\times}10^{-3}$
& $3967.3\pm417.9$ & $2.8{\times}10^{-5}\pm1.5{\times}10^{-5}$ \\

iCEM
& $\mathcal O(PK\,C_{\mathrm{eq}})$
& $\mathcal O(P\,n_\Theta)$
& $2332.2\pm126.7$ & $1.5{\times}10^{-2}\pm1.5{\times}10^{-2}$
& $3208.0\pm13.7$ & $9.8{\times}10^{-2}\pm3.6{\times}10^{-2}$
& $9907.3\pm59.2$ & $1.2{\times}10^{-2}\pm7.3{\times}10^{-3}$ \\

Adjoint-only
& $\mathcal O(K(C_{\mathrm{eq}}+C_{\mathrm{lin}}))$
& $\mathcal O\!\left(K(n_x+n_z)+n_\Theta\right)$
& $176.0\pm53.1$ & $3.2{\times}10^{-7}\pm5.6{\times}10^{-7}$
& $1111.2\pm83.1$ & $4.3{\times}10^{-4}\pm5.4{\times}10^{-4}$
& $935.2\pm164.7$ & $8.4{\times}10^{-3}\pm1.2{\times}10^{-2}$ \\

\textbf{Adjoint+RHC}
& $\mathcal O(H(C_{\mathrm{eq}}+C_{\mathrm{lin}}))$
& $\mathcal O\!\left(H(n_x+n_z)+n_\Theta\right)$
& $\mathbf{16.1\pm2.8}$ & $\mathbf{2.3{\times}10^{-7}\pm3.0{\times}10^{-7}}$
& $\mathbf{186.9\pm24.8}$ & $\mathbf{3.6{\times}10^{-8}\pm3.6{\times}10^{-8}}$
& $\mathbf{50.9\pm6.1}$ & $\mathbf{3.8{\times}10^{-8}\pm7.4{\times}10^{-9}}$ \\
\hline
\end{tabular}%
}
\caption{\textbf{Theory and empirical efficiency comparison.}
Here $K$ is the number of continuation steps in a full-horizon rollout,
$H$ is the number of steps per RHC segment, $P$ is the iCEM population
size, and $n_x,n_z,n_\Theta$ denote the dimensions of $\mathbf x$,
$\mathbf z$, and the neural-controller parameters $\Theta$, respectively.
$C_{\mathrm{eq}}$ is the cost of computing one converged equilibrium at a
continuation step, and $C_{\mathrm{lin}}$ is the cost of one additional
linearized tangent/adjoint solve with $\mathbf G_x$ or
$\mathbf G_x^\top$ used by implicit sensitivities. Memory for SPSA/iCEM
counts optimizer-side storage beyond the working memory required by a
single equilibrium solve. For adjoint methods, memory includes storing
the realized trajectory states for the backward proxy-adjoint pass. For
Adjoint+RHC, the listed theory cost is per optimizer update within one
RHC segment; the reported wall-clock time sums all segment optimizations
over the executed trajectory. The table reports total wall-clock runtime
and best loss achieved over the fixed update budget.}
\label{tab:theory_empirical}
\end{table*}

\subsection{Real-World Experiments}
\begin{figure*}[ht]
  \vskip 0.2in
\begin{center}\centerline{\includegraphics[width=1.0\textwidth]{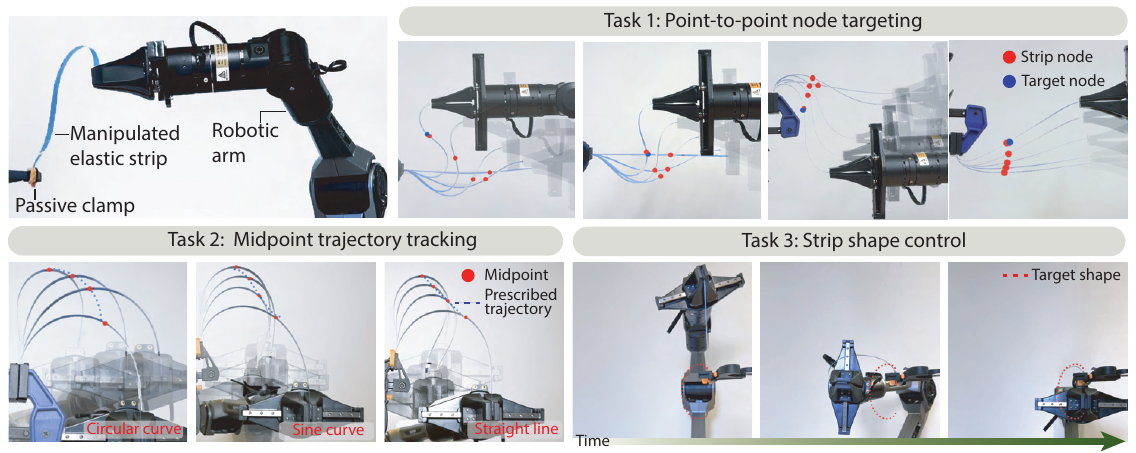}}
    \caption{Experimental setup and qualitative results for all three tasks.
    For Task~1 and Task~2, each subfigure shows sampled frames from the control process, illustrating point-to-point manipulation and midpoint trajectory following, respectively.
    For Task~3, the figure visualizes the initial, intermediate, and final configurations during the shape control process.}
    \label{fig:exp}
  \end{center}
\end{figure*}
As shown in Figure~\ref{fig:exp}, a 6-DoF robot arm actuates one endpoint of an elastic strip, while the other endpoint is fixed using a passive clamp.
The strip has measured length $\approx25$\,cm, width $0.5$\,cm, and thickness $0.127$\,mm.
Clamping one end yields an effective free length of $22$\,cm for Tasks~1--2 and $15$\,cm for Task~3, chosen to fit within the robot workspace.
At each step, the policy outputs a desired boundary increment for the actuated endpoint, which we convert into a Cartesian end-effector command.
For Tasks~1--2, the end-effector is constrained to translate in the $x$--$y$ plane; for Task~3, we additionally command an endpoint rotation.
Figure~\ref{fig:exp} shows representative hardware rollouts.
Across all three tasks, the observed strip motion matches the task requirements, indicating qualitative sim-to-real transfer under gravity and experimentally identified material parameters.
Additional intermediate frames and videos are available on the project page\textsuperscript{\ref{fn:projectpage}}.

\section{Discussion}
\label{sec:discussion}

Our results highlight two recurring difficulties in long-horizon optimization through sequential implicit computation. 
First, each state is defined by a converged equilibrium solve rather than an explicit transition map, making direct backpropagation through all Newton--CG iterations memory- and compute-intensive. 
Neural Control addresses this by keeping the forward rollout solver-consistent, with each state solved to the prescribed equilibrium tolerance, while computing proxy-adjoint gradients from local implicit sensitivities. 
Second, multistable equilibrium systems are branch-dependent: small actuation changes or accumulated gradient errors can move a warm-started solver to an unintended basin of attraction. 
RHC mitigates this issue by acting as a continuation and branch-tracking mechanism, repeatedly re-anchoring optimization to realized equilibria and localizing proxy-gradient errors to short segments.

Although our main experiments focus on elastic-strip manipulation, the framework applies more broadly to differentiable implicit models whose states are obtained from converged solves and whose local sensitivities can be computed through the equilibrium residual. 
This includes analytical mechanics models, differentiable equilibrium simulators, and learned implicit equilibrium models, as illustrated by our DEQ-style validation.

The method has several limitations. 
It assumes a quasi-static regime and does not model inertia or fast transient dynamics. 
The derivation also assumes a locally smooth equilibrium branch with a nonsingular and reasonably conditioned $\mathbf G_x$; near bifurcations, singular equilibria, or high-curvature branches, the frozen-tangent approximation may become less accurate. 
In addition, performance depends on the quality of the differentiable implicit model, and unmodeled non-conservative effects such as friction, hysteresis, material damping, or intermittent hard contact can degrade transfer. 
Finally, equilibrium solves remain the dominant computational cost, so scaling to high-dimensional 2D/3D deformable bodies will require efficient solvers, preconditioning, warm starts, and adaptive continuation.

\section{Conclusion}
\label{sec:conclusion}

We presented Neural Control, a proxy-adjoint framework for learning and control through sequential implicit equilibria.
The method keeps forward rollouts on converged equilibrium states while avoiding backpropagation through the inner Newton--CG iterations, and combines the resulting trajectory-dependent sensitivities with receding-horizon continuation for robust branch tracking in multistable landscapes. 
Empirically, Adjoint+RHC achieves large efficiency gains across elastic-strip manipulation tasks: it is up to 195$\times$ faster than the main SPSA/iCEM baselines while reaching substantially lower best loss under the same fixed update budgets. 
Hardware experiments further demonstrate qualitative sim-to-real transfer on a 6-DoF robotic setup under gravity and experimentally identified material parameters. 
Together, these results suggest that proxy-adjoint continuation is a practical gradient pathway for optimizing compliant, multistable systems whose states are defined by implicit equilibrium solves.

\section*{Acknowledgements}
We acknowledge financial support from the National Institutes of Health (grant 1R01NS141171-01), the National Science Foundation (grant 2332555), and the National Science Foundation (grant 2332554).

\section*{Impact Statement}
This work develops methods for learning and control in equilibrium-defined
physical systems, with a focus on deformable object manipulation and
compliant structures. Potential positive impacts include more efficient
simulation-based control, improved robotic manipulation of flexible
objects, and broader use of differentiable implicit models in physical AI.
Potential risks arise if such controllers are deployed on hardware without
adequate validation, especially under model mismatch, unmodeled dynamics,
contact, or safety-critical constraints. We encourage deployment with
appropriate hardware limits, task-specific validation, safety constraints,
and human oversight.


\bibliography{example_paper}

@inproceedings{mora2021pods,
  title={Pods: Policy optimization via differentiable simulation},
  author={Mora, Miguel Angel Zamora and Peychev, Momchil and Ha, Sehoon and Vechev, Martin and Coros, Stelian},
  booktitle={International Conference on Machine Learning},
  pages={7805--7817},
  year={2021},
  organization={PMLR}
}

@article{li2023dexdeform,
  title={Dexdeform: Dexterous deformable object manipulation with human demonstrations and differentiable physics},
  author={Li, Sizhe and Huang, Zhiao and Chen, Tao and Du, Tao and Su, Hao and Tenenbaum, Joshua B and Gan, Chuang},
  journal={arXiv preprint arXiv:2304.03223},
  year={2023}
}

@article{tong2024sim2real,
  title={Sim2real neural controllers for physics-based robotic deployment of deformable linear objects},
  author={Tong, Dezhong and Choi, Andrew and Qin, Longhui and Huang, Weicheng and Joo, Jungseock and Jawed, Mohammad Khalid},
  journal={The International Journal of Robotics Research},
  volume={43},
  number={6},
  pages={791--810},
  year={2024},
  publisher={SAGE Publications Sage UK: London, England}
}

@article{qiao2020scalable,
  title={Scalable differentiable physics for learning and control},
  author={Qiao, Yi-Ling and Liang, Junbang and Koltun, Vladlen and Lin, Ming C},
  journal={arXiv preprint arXiv:2007.02168},
  year={2020}
}

@article{huang2025tutorial,
  title={A tutorial on simulating nonlinear behaviors of flexible structures with the discrete differential geometry (DDG) method},
  author={Huang, Weicheng and Hao, Zhuonan and Li, Jiahao and Tong, Dezhong and Guo, Kexin and Zhang, Yingchao and Gao, Huajian and Hsia, K Jimmy and Liu, Mingchao},
  journal={Applied Mechanics Reviews},
  pages={1--88},
  year={2025}
}

@article{bretl2014quasi,
  title={Quasi-static manipulation of a Kirchhoff elastic rod based on a geometric analysis of equilibrium configurations},
  author={Bretl, Timothy and McCarthy, Zoe},
  journal={The International Journal of Robotics Research},
  volume={33},
  number={1},
  pages={48--68},
  year={2014},
  publisher={SAGE Publications Sage UK: London, England}
}

@article{tong2021automated,
  title={Automated stability testing of elastic rods with helical centerlines using a robotic system},
  author={Tong, Dezhong and Borum, Andy and Jawed, Mohammad Khalid},
  journal={IEEE Robotics and Automation Letters},
  volume={7},
  number={2},
  pages={1126--1133},
  year={2021},
  publisher={IEEE}
}

@article{yin2021modeling,
  title={Modeling, learning, perception, and control methods for deformable object manipulation},
  author={Yin, Hang and Varava, Anastasia and Kragic, Danica},
  journal={Science Robotics},
  volume={6},
  number={54},
  pages={eabd8803},
  year={2021},
  publisher={American Association for the Advancement of Science}
}

@article{moll2006path,
  title={Path planning for deformable linear objects},
  author={Moll, Mark and Kavraki, Lydia E},
  journal={IEEE Transactions on Robotics},
  volume={22},
  number={4},
  pages={625--636},
  year={2006},
  publisher={IEEE}
}

@incollection{bergou2008discrete,
  title={Discrete elastic rods},
  author={Bergou, Mikl{\'o}s and Wardetzky, Max and Robinson, Stephen and Audoly, Basile and Grinspun, Eitan},
  booktitle={ACM SIGGRAPH 2008 papers},
  pages={1--12},
  year={2008}
}

@inproceedings{suh2022differentiable,
  title={Do differentiable simulators give better policy gradients?},
  author={Suh, Hyung Ju and Simchowitz, Max and Zhang, Kaiqing and Tedrake, Russ},
  booktitle={International Conference on Machine Learning},
  pages={20668--20696},
  year={2022},
  organization={PMLR}
}

@inproceedings{hausknecht2015deep,
  title={Deep Recurrent Q-Learning for Partially Observable MDPs.},
  author={Hausknecht, Matthew J and Stone, Peter},
  booktitle={AAAI fall symposia},
  volume={45},
  pages={141},
  year={2015}
}

@article{georgiev2024adaptive,
  title={Adaptive horizon actor-critic for policy learning in contact-rich differentiable simulation},
  author={Georgiev, Ignat and Srinivasan, Krishnan and Xu, Jie and Heiden, Eric and Garg, Animesh},
  journal={arXiv preprint arXiv:2405.17784},
  year={2024}
}

@inproceedings{gao2024adaptive,
  title={Adaptive-gradient policy optimization: Enhancing policy learning in non-smooth differentiable simulations},
  author={Gao, Feng and Shi, Liangzhi and Zhang, Shenao and Wang, Zhaoran and Wu, Yi},
  booktitle={Forty-first International Conference on Machine Learning},
  year={2024}
}

@article{bishop1994mixture,
  title={Mixture density networks},
  author={Bishop, Christopher M},
  year={1994},
  publisher={Aston University}
}

@inproceedings{hafner2019learning,
  title={Learning latent dynamics for planning from pixels},
  author={Hafner, Danijar and Lillicrap, Timothy and Fischer, Ian and Villegas, Ruben and Ha, David and Lee, Honglak and Davidson, James},
  booktitle={International conference on machine learning},
  pages={2555--2565},
  year={2019},
  organization={PMLR}
}

@article{hafner2019dream,
  title={Dream to control: Learning behaviors by latent imagination},
  author={Hafner, Danijar and Lillicrap, Timothy and Ba, Jimmy and Norouzi, Mohammad},
  journal={arXiv preprint arXiv:1912.01603},
  year={2019}
}

@article{williams1990efficient,
  title={An efficient gradient-based algorithm for on-line training of recurrent network trajectories},
  author={Williams, Ronald J and Peng, Jing},
  journal={Neural computation},
  volume={2},
  number={4},
  pages={490--501},
  year={1990},
  publisher={MIT Press One Rogers Street, Cambridge, MA 02142-1209, USA journals-info~…}
}

@article{bai2019deep,
  title={Deep equilibrium models},
  author={Bai, Shaojie and Kolter, J Zico and Koltun, Vladlen},
  journal={Advances in neural information processing systems},
  volume={32},
  year={2019}
}

@article{gruslys2016memory,
  title={Memory-efficient backpropagation through time},
  author={Gruslys, Audrunas and Munos, R{\'e}mi and Danihelka, Ivo and Lanctot, Marc and Graves, Alex},
  journal={Advances in neural information processing systems},
  volume={29},
  year={2016}
}

@inproceedings{amos2017optnet,
  title={Optnet: Differentiable optimization as a layer in neural networks},
  author={Amos, Brandon and Kolter, J Zico},
  booktitle={International conference on machine learning},
  pages={136--145},
  year={2017},
  organization={PMLR}
}

@article{chen2018neural,
  title={Neural ordinary differential equations},
  author={Chen, Ricky TQ and Rubanova, Yulia and Bettencourt, Jesse and Duvenaud, David K},
  journal={Advances in neural information processing systems},
  volume={31},
  year={2018}
}

@article{chua2018deep,
  title={Deep reinforcement learning in a handful of trials using probabilistic dynamics models},
  author={Chua, Kurtland and Calandra, Roberto and McAllister, Rowan and Levine, Sergey},
  journal={Advances in neural information processing systems},
  volume={31},
  year={2018}
}

@article{spall2002multivariate,
  title={Multivariate stochastic approximation using a simultaneous perturbation gradient approximation},
  author={Spall, James C},
  journal={IEEE transactions on automatic control},
  volume={37},
  number={3},
  pages={332--341},
  year={2002},
  publisher={IEEE}
}

@incollection{botev2013cross,
  title={The cross-entropy method for optimization},
  author={Botev, Zdravko I and Kroese, Dirk P and Rubinstein, Reuven Y and L’ecuyer, Pierre},
  booktitle={Handbook of statistics},
  volume={31},
  pages={35--59},
  year={2013},
  publisher={Elsevier}
}

@article{salimans2017evolution,
  title={Evolution strategies as a scalable alternative to reinforcement learning},
  author={Salimans, Tim and Ho, Jonathan and Chen, Xi and Sidor, Szymon and Sutskever, Ilya},
  journal={arXiv preprint arXiv:1703.03864},
  year={2017}
}

@article{mania2018simple,
  title={Simple random search of static linear policies is competitive for reinforcement learning},
  author={Mania, Horia and Guy, Aurelia and Recht, Benjamin},
  journal={Advances in neural information processing systems},
  volume={31},
  year={2018}
}

@inproceedings{amos2020differentiable,
  title={The differentiable cross-entropy method},
  author={Amos, Brandon and Yarats, Denis},
  booktitle={International Conference on Machine Learning},
  pages={291--302},
  year={2020},
  organization={PMLR}
}

@inproceedings{pinneri2021sample,
  title={Sample-efficient cross-entropy method for real-time planning},
  author={Pinneri, Cristina and Sawant, Shambhuraj and Blaes, Sebastian and Achterhold, Jan and Stueckler, Joerg and Rolinek, Michal and Martius, Georg},
  booktitle={Conference on Robot Learning},
  pages={1049--1065},
  year={2021},
  organization={PMLR}
}

@inproceedings{amos2018differentiable,
  title     = {Differentiable {MPC} for End-to-End Planning and Control},
  author    = {Amos, Brandon and Jimenez, Ivan and Sacks, Jacob and Boots, Byron and Kolter, J. Zico},
  booktitle = {Advances in Neural Information Processing Systems},
  volume    = {31},
  year      = {2018}
}

@article{jin2020pontryagin,
  title={Pontryagin differentiable programming: An end-to-end learning and control framework},
  author={Jin, Wanxin and Wang, Zhaoran and Yang, Zhuoran and Mou, Shaoshuai},
  journal={Advances in Neural Information Processing Systems},
  volume={33},
  pages={7979--7992},
  year={2020}
}

@article{hu2019difftaichi,
  title={Difftaichi: Differentiable programming for physical simulation},
  author={Hu, Yuanming and Anderson, Luke and Li, Tzu-Mao and Sun, Qi and Carr, Nathan and Ragan-Kelley, Jonathan and Durand, Fr{\'e}do},
  journal={arXiv preprint arXiv:1910.00935},
  year={2019}
}

@article{huang2021plasticinelab,
  title={Plasticinelab: A soft-body manipulation benchmark with differentiable physics},
  author={Huang, Zhiao and Hu, Yuanming and Du, Tao and Zhou, Siyuan and Su, Hao and Tenenbaum, Joshua B and Gan, Chuang},
  journal={arXiv preprint arXiv:2104.03311},
  year={2021}
}

@inproceedings{lin2021softgym,
  title={Softgym: Benchmarking deep reinforcement learning for deformable object manipulation},
  author={Lin, Xingyu and Wang, Yufei and Olkin, Jake and Held, David},
  booktitle={Conference on Robot Learning},
  pages={432--448},
  year={2021},
  organization={PMLR}
}

@article{freeman2021brax,
  title={Brax--a differentiable physics engine for large scale rigid body simulation},
  author={Freeman, C Daniel and Frey, Erik and Raichuk, Anton and Girgin, Sertan and Mordatch, Igor and Bachem, Olivier},
  journal={arXiv preprint arXiv:2106.13281},
  year={2021}
}

@article{qiao2021differentiable,
  title={Differentiable simulation of soft multi-body systems},
  author={Qiao, Yiling and Liang, Junbang and Koltun, Vladlen and Lin, Ming},
  journal={Advances in Neural Information Processing Systems},
  volume={34},
  pages={17123--17135},
  year={2021}
}

@inproceedings{hu2019chainqueen,
  title={Chainqueen: A real-time differentiable physical simulator for soft robotics},
  author={Hu, Yuanming and Liu, Jiancheng and Spielberg, Andrew and Tenenbaum, Joshua B and Freeman, William T and Wu, Jiajun and Rus, Daniela and Matusik, Wojciech},
  booktitle={2019 International conference on robotics and automation (ICRA)},
  pages={6265--6271},
  year={2019},
  organization={IEEE}
}

@article{du2021diffpd,
  title={Diffpd: Differentiable projective dynamics},
  author={Du, Tao and Wu, Kui and Ma, Pingchuan and Wah, Sebastien and Spielberg, Andrew and Rus, Daniela and Matusik, Wojciech},
  journal={ACM Transactions on Graphics (ToG)},
  volume={41},
  number={2},
  pages={1--21},
  year={2021},
  publisher={ACM New York, NY}
}

@inproceedings{grinspun2003discrete,
author = {Grinspun, Eitan and Hirani, Anil N. and Desbrun, Mathieu and Schr\"{o}der, Peter},
title = {Discrete shells},
year = {2003},
isbn = {1581136595},
publisher = {Eurographics Association},
address = {Goslar, DEU},
booktitle = {Proceedings of the 2003 ACM SIGGRAPH/Eurographics Symposium on Computer Animation},
pages = {62–67},
numpages = {6},
location = {San Diego, California},
series = {SCA '03}
}

@inproceedings{sanchez2020gns,
  author    = {Alvaro Sanchez{-}Gonzalez and Jonathan Godwin and Tobias Pfaff and Rex Ying and Jure Leskovec and Peter W. Battaglia},
  title     = {Learning to Simulate Complex Physics with Graph Networks},
  booktitle = {Proceedings of the 37th International Conference on Machine Learning (ICML)},
  series    = {Proceedings of Machine Learning Research},
  volume    = {119},
  pages     = {8459--8468},
  year      = {2020},
  publisher = {PMLR},
  url       = {http://proceedings.mlr.press/v119/sanchez-gonzalez20a.html}
}

@inproceedings{pfaff2021meshgraphnets,
  author    = {Tobias Pfaff and Meire Fortunato and Alvaro Sanchez{-}Gonzalez and Peter W. Battaglia},
  title     = {Learning Mesh-Based Simulation with Graph Networks},
  booktitle = {9th International Conference on Learning Representations (ICLR)},
  year      = {2021},
  publisher = {OpenReview.net},
  url       = {https://openreview.net/forum?id=roNqYL0\_XP}
}

@inproceedings{xing2025sapo,
  author    = {Eliot Xing and Vernon Luk and Jean Oh},
  title     = {Stabilizing Reinforcement Learning in Differentiable Multiphysics Simulation},
  booktitle = {The Thirteenth International Conference on Learning Representations (ICLR)},
  year      = {2025},
  publisher = {OpenReview.net},
  url       = {https://openreview.net/forum?id=DRiLWb8bJg}
}

@article{agrawal2019differentiable,
  title={Differentiable convex optimization layers},
  author={Agrawal, Akshay and Amos, Brandon and Barratt, Shane and Boyd, Stephen and Diamond, Steven and Kolter, J Zico},
  journal={Advances in neural information processing systems},
  volume={32},
  year={2019}
}

@article{tong2025discrete,
  title={Discrete differential geometry for simulating nonlinear behaviors of flexible systems: A survey},
  author={Tong, Dezhong and Choi, Andrew and Wang, Jiaqi and Huang, Weicheng and Chen, Zexiong and Li, Jiahao and Huang, Xiaonan and Liu, Mingchao and Gao, Huajian and Hsia, K Jimmy},
  journal={Extreme Mechanics Letters},
  pages={102430},
  year={2025},
  publisher={Elsevier}
}
\bibliographystyle{icml2026}



\newpage
\appendix
\onecolumn

\section{Discrete Proxy-Adjoint Derivation and Implementation Details}
\label{app:discrete_proxy_adjoint}

This appendix derives the discrete proxy-adjoint recursions used in
Section~\ref{sec:proxy_adjoints} and describes the implementation used
to compute policy gradients without differentiating through the inner
equilibrium solver. The key primitive is a matrix-free routine for
products $\mathbf S_k^\top\mathbf v$, where
$\mathbf S_k=\partial\mathbf x_k/\partial\mathbf z_k$ is the local
equilibrium sensitivity at continuation step $k$.

\paragraph{Forward rollout and notation.}

We discretize $\lambda\in[0,1]$ into
$\{\lambda_k\}_{k=0}^{K}$ with step size
$\Delta\lambda=1/K$. At each step $k=0,\ldots,K-1$, the controller
outputs
\begin{equation}
\mathbf u_k=\mathbf u_\Theta(\lambda_k),
\end{equation}
and the continuation variable is updated explicitly as
\begin{equation}
\mathbf z_{k+1}
=
\mathbf z_k+\Delta\lambda\,f(\mathbf z_k,\mathbf u_k).
\label{eq:app:z_update}
\end{equation}
The state $\mathbf x_k$ is defined as the converged equilibrium at
$\mathbf z_k$, warm-started from the previous solution:
\begin{equation}
\mathbf x_k
=
\mathrm{SolveEq}(\mathbf z_k;\ \mathrm{init}=\mathbf x_{k-1}),
\qquad
\mathcal G(\mathbf x_k,\mathbf z_k)=\mathbf 0 .
\label{eq:app:exact_equilibrium}
\end{equation}
The forward rollout always uses these converged equilibria. The proxy dynamics introduced below is used only for the backward sensitivity
calculation.

Along the realized rollout, define the equilibrium Jacobians
\begin{equation}
\mathbf G_{x,k}:=\left.\frac{\partial\mathcal G}{\partial\mathbf x}\right|_{(\mathbf x_k,\mathbf z_k,\lambda_k)},
\qquad
\mathbf G_{z,k}:=\left.\frac{\partial\mathcal G}{\partial\mathbf z}\right|_{(\mathbf x_k,\mathbf z_k,\lambda_k)},
\label{eq:app:GxGz_defs}
\end{equation}
and the continuation-dynamics Jacobians
\begin{equation}
\mathbf A_k:=\left.\frac{\partial f}{\partial\mathbf z}\right|_{(\mathbf z_k,\mathbf u_k)},
\qquad
\mathbf B_k:=\left.\frac{\partial f}{\partial\mathbf u}\right|_{(\mathbf z_k,\mathbf u_k)}.
\label{eq:app:AB_defs}
\end{equation}

\paragraph{Implicit sensitivity and $\mathbf S_k^\top\mathbf v$ products.}
Assuming that the realized branch is locally smooth and that
$\mathbf G_{x,k}$ is nonsingular, implicit differentiation of
$\mathcal G(\mathbf x_k,\mathbf z_k)=\mathbf 0$ gives
\begin{equation}
\mathbf S_k
:=
\frac{\partial\mathbf x_k}{\partial\mathbf z_k},
\qquad
\mathbf G_{x,k}\mathbf S_k=-\mathbf G_{z,k}.
\label{eq:app:S_def}
\end{equation}
We never form $\mathbf S_k$ explicitly. Instead, for any vector
$\mathbf v$, we compute $\mathbf S_k^\top\mathbf v$ by solving the
discrete adjoint linear system
\begin{equation}
\mathbf G_{x,k}^\top\boldsymbol\eta_k=\mathbf v,
\label{eq:app:adjoint_solve}
\end{equation}
and returning
\begin{equation}
\mathbf S_k^\top\mathbf v
=
-\mathbf G_{z,k}^\top\boldsymbol\eta_k .
\label{eq:app:Stv}
\end{equation}
This matrix-free product is the only place where equilibrium
sensitivities enter the gradient computation.

\paragraph{Proxy dynamics and frozen tangent}
Since $\mathbf x_k$ is an implicit function of $\mathbf z_k$, the
first-order branch-consistent variation satisfies
\begin{equation}
\delta\mathbf x_k
\approx
\mathbf S_k\,\delta\mathbf z_k .
\end{equation}
Using the explicit continuation update~\eqref{eq:app:z_update}, we define
the discrete proxy transition
\begin{equation}
\mathbf x_{k+1}
\approx
\mathbf x_k
+
\Delta\lambda\,\mathbf S_k f(\mathbf z_k,\mathbf u_k).
\label{eq:app:x_proxy}
\end{equation}
Equation~\eqref{eq:app:x_proxy} is not used to generate the forward
trajectory; the forward pass still solves
Eq.~\eqref{eq:app:exact_equilibrium} to convergence. Instead,
Eq.~\eqref{eq:app:x_proxy} defines a tractable backward pass.

The frozen-tangent approximation treats $\mathbf S_k$ as locally constant
during the backward pass. Equivalently, we ignore derivatives of
$\mathbf S_k$ with respect to $\mathbf x_k$, $\mathbf z_k$, and
$\lambda_k$. This avoids higher-order derivatives of the equilibrium
residual. Receding-horizon continuation limits the accumulation of this
proxy mismatch by restricting the backward propagation horizon
(Section~\ref{sec:receding_horizon_control}).

\paragraph{Unified discrete objective.}
We write terminal and trajectory objectives in the form
\begin{equation}
\mathcal L
=
\phi(\mathbf x_K,\mathbf z_K)
+
\sum_{k=0}^{K-1}
w_k\,\ell(\mathbf x_k,\mathbf z_k),
\label{eq:app:unified_loss}
\end{equation}
where $\phi$ is an optional terminal loss and $\ell$ is an optional
running loss. Terminal-only objectives use $w_k\equiv0$; trajectory
objectives set $\phi\equiv0$ if no terminal loss is needed. The scalar
$w_k$ includes the quadrature weight; for a Riemann approximation of
$\int_0^1 w(\lambda)\ell\,d\lambda$, one may take
$w_k=\Delta\lambda\,w(\lambda_k)$.

\paragraph{Proxy adjoints.}
Define proxy-system adjoints
\begin{equation}
\mathbf a_k
:=
\frac{\partial\mathcal L}{\partial\mathbf x_k},
\qquad
\mathbf g_k
:=
\frac{\partial\mathcal L}{\partial\mathbf z_k}.
\end{equation}
The terminal conditions are
\begin{equation}
\mathbf a_K
=
\nabla_{\mathbf x}\phi(\mathbf x_K,\mathbf z_K),
\qquad
\mathbf g_K
=
\nabla_{\mathbf z}\phi(\mathbf x_K,\mathbf z_K).
\label{eq:app:terminal_bc_unified}
\end{equation}
The dependence of the equilibrium state on the continuation variable is
accounted for through $\mathbf S_k^\top\mathbf a_{k+1}$ in the backward
recursion and control gradient below.

Under the proxy transition~\eqref{eq:app:x_proxy} and frozen tangent, the
linearized transitions are
\begin{equation}
\frac{\partial\mathbf z_{k+1}}{\partial\mathbf z_k}
=
\mathbf I+\Delta\lambda\,\mathbf A_k,
\qquad
\frac{\partial\mathbf z_{k+1}}{\partial\mathbf u_k}
=
\Delta\lambda\,\mathbf B_k,
\label{eq:app:dz_jacs}
\end{equation}
and
\begin{equation}
\frac{\partial\mathbf x_{k+1}}{\partial\mathbf x_k}
=
\mathbf I,
\qquad
\frac{\partial\mathbf x_{k+1}}{\partial\mathbf z_k}
\approx
\Delta\lambda\,\mathbf S_k\mathbf A_k,
\qquad
\frac{\partial\mathbf x_{k+1}}{\partial\mathbf u_k}
\approx
\Delta\lambda\,\mathbf S_k\mathbf B_k .
\label{eq:app:dx_jacs}
\end{equation}
Therefore, for $k=K-1,\ldots,0$, the adjoints satisfy
\begin{equation}
\mathbf a_k
\approx
\mathbf a_{k+1}
+
w_k\,\nabla_{\mathbf x}\ell(\mathbf x_k,\mathbf z_k),
\label{eq:app:a_rec_unified}
\end{equation}
and
\begin{equation}
\begin{aligned}
\mathbf g_k
\approx\;&
(\mathbf I+\Delta\lambda\,\mathbf A_k)^\top\mathbf g_{k+1}
+
\Delta\lambda\,\mathbf A_k^\top
\big(\mathbf S_k^\top\mathbf a_{k+1}\big)
\\
&+
w_k\,\nabla_{\mathbf z}\ell(\mathbf x_k,\mathbf z_k).
\end{aligned}
\label{eq:app:g_rec_unified}
\end{equation}
Each product $\mathbf S_k^\top(\cdot)$ is evaluated using
Eq.~\eqref{eq:app:Stv}. Notice that
$\mathbf S_k^\top\nabla_{\mathbf x}\ell$ does not appear separately in
Eq.~\eqref{eq:app:g_rec_unified}; the running-state loss contribution is
first accumulated in $\mathbf a_k$ through
Eq.~\eqref{eq:app:a_rec_unified}, and then converted to a
$\mathbf z$-space contribution through $\mathbf S_k^\top\mathbf a$.

For convenience define the effective continuation adjoint
\begin{equation}
\mathbf d_k
:=
\mathbf g_{k+1}
+
\mathbf S_k^\top\mathbf a_{k+1}.
\label{eq:app:d_def}
\end{equation}
The control gradient at each step is
\begin{equation}
\frac{d\mathcal L}{d\mathbf u_k}
\approx
\Delta\lambda\,\mathbf B_k^\top\mathbf d_k.
\label{eq:app:dL_du_unified}
\end{equation}
The policy-parameter gradient then follows by the chain rule through the
controller:
\begin{equation}
\frac{d\mathcal L}{d\Theta}
=
\sum_{k=0}^{K-1}
\left(
\frac{\partial\mathbf u_\Theta(\lambda_k)}{\partial\Theta}
\right)^\top
\frac{d\mathcal L}{d\mathbf u_k}.
\label{eq:app:dL_dTheta}
\end{equation}

\paragraph{Implementation.}
Algorithm~\ref{alg:adjoint_rhc} summarizes the implementation of
proxy-adjoint training with receding-horizon continuation.

\begin{algorithm}[!h]
  \caption{Adjoint+RHC: proxy-adjoint training with receding-horizon continuation}
  \label{alg:adjoint_rhc}
  \begin{algorithmic}
    \STATE {\bfseries Input:} segment horizon $H$, number of segments $M$, step size $\Delta\lambda$
    \STATE {\bfseries Input:} initial $(\mathbf x^{(0)},\mathbf z^{(0)})$, controller $\mathbf u_\Theta(\lambda)$, continuation dynamics $f(\mathbf z,\mathbf u)$
    \STATE {\bfseries Input:} segment update budget $T_{\mathrm{seg}}$, optimizer step $\mathrm{OptStep}(\Theta, d\mathcal L/d\Theta)$
    \STATE {\bfseries Helper:} $\mathrm{StProd}(k,\mathbf v)$ computes $\mathbf S_k^\top\mathbf v$:
    \STATE \hspace{1.5em} solve $\mathbf G_{x,k}^\top\boldsymbol\eta_k=\mathbf v$, return $-\mathbf G_{z,k}^\top\boldsymbol\eta_k$
    \STATE
    \FOR{$m=0$ {\bfseries to} $M-1$}
      \STATE Set segment start $(\mathbf x_0,\mathbf z_0)\gets(\mathbf x^{(m)},\mathbf z^{(m)})$
      \STATE Initialize $\Theta\gets\Theta_0$
      \FOR{$t=1$ {\bfseries to} $T_{\mathrm{seg}}$}
        \STATE {\itshape Forward rollout with converged equilibria.}
        \FOR{$k=0$ {\bfseries to} $H-1$}
          \STATE $\lambda_k^{(m)}\gets (mH+k)\Delta\lambda$
          \STATE $\mathbf u_k\gets\mathbf u_\Theta(\lambda_k^{(m)})$
          \STATE Cache $\mathbf G_{x,k},\mathbf G_{z,k},\mathbf A_k,\mathbf B_k$ at $(\mathbf x_k,\mathbf z_k,\mathbf u_k)$
          \STATE $\mathbf z_{k+1}\gets\mathbf z_k+\Delta\lambda\,f(\mathbf z_k,\mathbf u_k)$
          \STATE $\mathbf x_{k+1}\gets\mathrm{SolveEq}(\mathbf z_{k+1};\ \mathrm{init}=\mathbf x_k)$
        \ENDFOR
        \STATE Compute segment loss $\mathcal L$ on the realized segment rollout
        \STATE {\itshape Backward pass with proxy adjoints and frozen tangent.}
        \STATE Initialize $\mathbf a_H\gets\partial\mathcal L/\partial\mathbf x_H$
        \STATE Initialize $\mathbf g_H\gets\partial\mathcal L/\partial\mathbf z_H$
        \FOR{$k=H-1$ {\bfseries to} $0$}
          \STATE $\mathbf s_k\gets\mathrm{StProd}(k,\mathbf a_{k+1})$
          \STATE $\mathbf d_k\gets\mathbf g_{k+1}+\mathbf s_k$
          \STATE $d\mathcal L/d\mathbf u_k\gets\Delta\lambda\,\mathbf B_k^\top\mathbf d_k$
          \STATE $\mathbf a_k\gets\mathbf a_{k+1}+w_k\,\nabla_{\mathbf x}\ell(\mathbf x_k,\mathbf z_k)$
          \STATE $\mathbf g_k\gets(\mathbf I+\Delta\lambda\,\mathbf A_k)^\top\mathbf g_{k+1}
          +\Delta\lambda\,\mathbf A_k^\top\mathbf s_k
          +w_k\,\nabla_{\mathbf z}\ell(\mathbf x_k,\mathbf z_k)$
        \ENDFOR
        \STATE $d\mathcal L/d\Theta\gets
        \sum_{k=0}^{H-1}
        \left(\partial\mathbf u_\Theta(\lambda_k^{(m)})/\partial\Theta\right)^\top
        (d\mathcal L/d\mathbf u_k)$
        \STATE $\Theta\gets\mathrm{OptStep}(\Theta,d\mathcal L/d\Theta)$
      \ENDFOR
      \STATE {\itshape Execute segment once with optimized $\Theta$.}
      \STATE Roll out from $(\mathbf x^{(m)},\mathbf z^{(m)})$ with optimized $\Theta$ to obtain terminal $(\mathbf x_H,\mathbf z_H)$
      \STATE Set $(\mathbf x^{(m+1)},\mathbf z^{(m+1)})\gets(\mathbf x_H,\mathbf z_H)$
    \ENDFOR
    \STATE {\bfseries Output:} concatenated executed rollout and task loss evaluated over all segments
  \end{algorithmic}
\end{algorithm}

\section{Elastic Strip Equilibrium Model and Implicit Solver}
\label{app:strip_solver}
This section instantiates the equilibrium residual
$\mathcal G(\mathbf x,\mathbf z)$ and the implicit solver
$\mathbf x_k=\mathrm{SolveEq}(\mathbf z_k)$ used in the elastic-strip
experiments.

\paragraph{State and controls.}
\begin{figure}
\centering\includegraphics[width=0.5\linewidth]{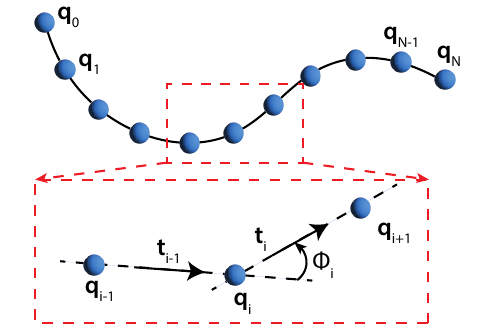}
  \caption{Discrete schematic of the elastic strip. The centerline is
  represented by nodal positions $\{\mathbf q_i\}$, edge vectors
  $\{\mathbf e_i\}$, unit tangents $\{\mathbf t_i\}$, and turning angles
  $\{\phi_i\}$.}
    \label{fig:discrete_schematic}
\end{figure}
As shown in Fig.~\ref{fig:discrete_schematic}, we discretize the strip centerline into $N{+}1$ nodes
$\{\mathbf q_i\}_{i=0}^{N}$ with $\mathbf q_i\in\mathbb R^d$ ($d{=}2$ or $3$), and stack them as
$\mathbf q=[\mathbf q_0^\top,\ldots,\mathbf q_N^\top]^\top$.
We clamp two nodes at each end and treat them as boundary (controllable) DoFs,
\begin{equation}
\begin{aligned}
\mathbf q_b &:= [\mathbf q_0^\top,\mathbf q_1^\top,\mathbf q_{N-1}^\top,\mathbf q_N^\top]^\top,\\
\mathbf q_f &:= [\mathbf q_2^\top,\ldots,\mathbf q_{N-2}^\top]^\top .
\end{aligned}
\end{equation}
We identify $\mathbf x \equiv \mathbf q_f$ and $\mathbf z \equiv \mathbf q_b$.

\paragraph{Energy model.}
Strip equilibrium is defined by a stationary point of the total potential energy
\begin{equation}
E_{\mathrm{tot}}(\mathbf q_f)
=
E_s(\mathbf q_f, \mathbf q_b)
+
E_b(\mathbf q_f, \mathbf q_b)
+
E_g(\mathbf q_f, \mathbf q_b),
\label{eq:app:Etot_strip}
\end{equation}
including stretching, bending, and gravity.\footnote{We do not model twist; the strip is represented by its centerline with stretching and bending energies.}

\paragraph{Stretching energy.}
Let $\Delta l$ be the undeformed edge length and $\mathbf e_i := \mathbf q_{i+1}-\mathbf q_i$ for $i=0,\ldots,N-1$.
We use
\begin{equation}
E_s
=
\frac{1}{2} k_s \sum_{i=0}^{N-1}
\left(
1 - \frac{\|\mathbf e_i\|}{\Delta l}
\right)^2 \Delta l .
\label{eq:stretching}
\end{equation}

\paragraph{Bending energy.}
Let $\phi_i$ be the turning angle at node $i$ (Fig.~\ref{fig:discrete_schematic}) and $\phi_i^0$ the rest angle (zero in our study).
We use
\begin{equation}
E_b
=
\frac{1}{2} k_b \sum_{i=1}^{N-1}
\left(
2\tan\frac{\phi_i}{2} - 2\tan\frac{\phi_i^0}{2}
\right)^2 \frac{1}{\Delta l}.
\label{eq:bending}
\end{equation}

\paragraph{Gravitational energy.}
We use a lumped-mass gravity potential
\begin{equation}
E_g
=
-\sum_{i=0}^{N} m_i\, \mathbf g^\top \mathbf q_i ,
\label{eq:app:gravity}
\end{equation}
where $\mathbf g\in\mathbb R^d$ is the gravitational acceleration and $m_i$ is the lumped nodal mass.

\paragraph{Equilibrium residual.}
Given $\mathbf q_b$, the free configuration returned by
$\mathrm{SolveEq}$ is a locally stable equilibrium selected by the
warm-started solver. Equivalently, it satisfies the stationarity condition
\begin{equation}
\nabla_{\mathbf q_f}E_{\mathrm{tot}}(\mathbf q_f^\star,\mathbf q_b)
=
\mathbf 0,
\label{eq:app:strip_equilibrium_stationary}
\end{equation}
with the particular local branch determined by the previous continuation
step. We therefore define
\begin{equation}
\begin{aligned}
&\mathcal G(\mathbf x,\mathbf z)
:=
\nabla_{\mathbf q_f} E_{\mathrm{tot}}(\mathbf q_f,\mathbf q_b), \\
\textrm{with}\qquad
&\mathbf x\equiv \mathbf q_f,\quad \mathbf z\equiv \mathbf q_b .
\end{aligned}
\label{eq:app:strip_G_def}
\end{equation}
The Jacobians $\mathbf G_{x,k}$ and $\mathbf G_{z,k}$ in
Eq.~\eqref{eq:app:GxGz_defs} are the corresponding derivatives of
$\mathcal G$ evaluated at $(\mathbf q_{f,k},\mathbf q_{b,k})$; namely,
$\mathbf G_{x,k}=\nabla^2_{\mathbf q_f\mathbf q_f}E_{\mathrm{tot}}$ and
$\mathbf G_{z,k}=\nabla^2_{\mathbf q_f\mathbf q_b}E_{\mathrm{tot}}$.

\paragraph{Implicit solver (\texttt{SolveEq}).}
We implement $\mathbf x_k=\mathrm{SolveEq}(\mathbf z_k,\lambda_k)$ using a trust-region Newton--CG method applied to the stationarity condition~\eqref{eq:app:strip_equilibrium_stationary}. At iterate $\mathbf q_f^{(t)}$, we assemble
$E_\mathrm{tot}^{(t)}$, $\mathbf g^{(t)}:=\nabla_{\mathbf q_f}E_{\mathrm{tot}}$, and
$\mathbf H^{(t)}:=\nabla_{\mathbf q_f}^2E_{\mathrm{tot}}$, and symmetrize
$\mathbf H^{(t)}\leftarrow\frac12(\mathbf H^{(t)}+\mathbf H^{(t)\top})$.

\paragraph{Trust-region step.}
We approximately solve
\begin{equation}
\min_{\|\mathbf p\|_2 \le \Delta^{(t)}}\;
m^{(t)}(\mathbf p)
:=
\mathbf g^{(t)\top}\mathbf p
+\frac{1}{2}\mathbf p^\top \mathbf H^{(t)}\mathbf p .
\label{eq:app:tr_subproblem}
\end{equation}
If $\mathbf H^{(t)}$ is SPD (via Cholesky-type factorization), we compute the Newton step
$\mathbf H^{(t)}\mathbf p_N^{(t)}=-\mathbf g^{(t)}$ and take $\mathbf p^{(t)}=\mathbf p_N^{(t)}$ if feasible;
otherwise we take a standard dogleg boundary step between the Cauchy and Newton directions.
If SPD detection fails, we use Steihaug PCG with Hessian-vector products, terminating on negative curvature
$\mathbf d^\top \mathbf H^{(t)}\mathbf d\le 0$ or on hitting $\|\mathbf p\|_2=\Delta^{(t)}$.
We use a Jacobi preconditioner
$\mathbf M^{(t)}=\mathrm{diag}(|\mathrm{diag}(\mathbf H^{(t)})|)+\sigma\mathbf I$.

\paragraph{Acceptance, radius update, and termination.}
Let $\mathrm{pred}^{(t)}=-m^{(t)}(\mathbf p^{(t)})$ and
$\mathrm{ared}^{(t)}=E^{(t)}-E_{\mathrm{tot}}(\mathbf q_f^{(t)}+\mathbf p^{(t)},\mathbf q_b;\lambda)$, and define
$\rho^{(t)}=\mathrm{ared}^{(t)}/\max(\mathrm{pred}^{(t)},\epsilon)$.
We accept the step if $\rho^{(t)}\ge \eta$, and update $\Delta^{(t)}$ using standard shrink/expand rules based on $\rho^{(t)}$.
We terminate when $\|\mathbf g^{(t)}\|_2\le \varepsilon_g$, $\Delta^{(t)}\le \varepsilon_\Delta$, or after a maximum number of iterations.
We warm-start continuation step $k$ with the previous equilibrium $\mathbf q_{f,k}^{(0)}\leftarrow \mathbf q_{f,k-1}^\star$;
under multistability, this initialization selects a specific locally stable equilibrium branch for the same $\mathbf z$.

\section{Assumptions and Consistency of the Frozen-Tangent Proxy Gradient}
\label{app:frozen_tangent}

This section summarizes the assumptions under which the frozen-tangent
proxy gradient used in Sec.~\ref{sec:proxy_adjoints} provides a
controlled local approximation to the gradient along the realized
equilibrium branch.

We consider the equilibrium constraint
\begin{equation}
\mathcal G(\mathbf x(\lambda),\mathbf z(\lambda))=\mathbf 0,
\qquad \lambda\in[0,1],
\label{eq:app_G}
\end{equation}
where $\mathbf z(\lambda)$ is generated by the continuation dynamics
$d\mathbf z/d\lambda=f(\mathbf z(\lambda),\mathbf u(\lambda))$.
Let
\[
\mathbf G_x=\frac{\partial\mathcal G}{\partial\mathbf x},
\qquad
\mathbf G_z=\frac{\partial\mathcal G}{\partial\mathbf z},
\qquad
\mathbf G_\lambda=\frac{\partial\mathcal G}{\partial\lambda}.
\]

\begin{assumption}[Smooth realized branch]
\label{assm:branch}
The residual $\mathcal G$ is $C^2$ in a neighborhood of the realized
rollout, and the warm-started solver follows a locally unique equilibrium
branch $\mathbf x(\lambda)\in C^1([0,1])$. Along this branch,
$\mathbf G_x(\lambda)$ is uniformly nonsingular:
\begin{equation}
\sup_{\lambda\in[0,1]}
\left\|\mathbf G_x(\lambda)^{-1}\right\| < \infty .
\end{equation}
This assumption excludes singular equilibria, loss of local stability,
and basin switches, where $\mathbf G_x$ may become ill-conditioned or
singular.
\end{assumption}

\paragraph{Implicit sensitivity.}
Under Assumption~\ref{assm:branch}, the implicit function theorem gives
a local equilibrium map $\mathbf x^\star(\mathbf z,\lambda)$. Its
sensitivity with respect to $\mathbf z$ is
\begin{equation}
\mathbf S(\lambda)
:=
\frac{\partial\mathbf x^\star}{\partial\mathbf z}(\lambda),
\qquad
\mathbf G_x(\lambda)\mathbf S(\lambda)
=
-\mathbf G_z(\lambda).
\label{eq:app_S}
\end{equation}
Differentiating Eq.~\eqref{eq:app_G} along the realized branch gives
\begin{equation}
\frac{d\mathbf x}{d\lambda}
=
\mathbf S(\lambda)\frac{d\mathbf z}{d\lambda}
-
\mathbf G_x(\lambda)^{-1}\mathbf G_\lambda(\lambda).
\label{eq:app_exact_dxdl}
\end{equation}
In our experiments, explicit $\lambda$-dependence is either absent or
absorbed into $\mathbf z(\lambda)$, so $\mathbf G_\lambda\equiv 0$ and
\begin{equation}
\frac{d\mathbf x}{d\lambda}
=
\mathbf S(\lambda) f(\mathbf z(\lambda),\mathbf u(\lambda)).
\label{eq:app_exact_dxdl_simplified}
\end{equation}

\paragraph{Frozen-tangent approximation.}
On a discretization
$0=\lambda_0<\lambda_1<\cdots<\lambda_K=1$, the backward pass evaluates
$\mathbf S_k:=\mathbf S(\lambda_k)$ on the realized rollout and treats it
as constant over the interval $[\lambda_k,\lambda_{k+1}]$. The local
tangent variation on this interval is
\begin{equation}
\eta_k
:=
\sup_{\lambda\in[\lambda_k,\lambda_{k+1}]}
\left\|
\mathbf S(\lambda)-\mathbf S_k
\right\|.
\label{eq:app_eta_k}
\end{equation}
Let
\[
\eta_S:=\max_k \eta_k .
\]
If $\mathbf S$ is Lipschitz continuous with constant $L_S$, then
\begin{equation}
\eta_S \le L_S\Delta\lambda,
\qquad
\Delta\lambda=\max_k(\lambda_{k+1}-\lambda_k).
\label{eq:app_lipschitz_S}
\end{equation}

\begin{proposition}[Local consistency of the frozen-tangent gradient]
\label{prop:proxy_consistency}
Consider the trajectory objective in Eq.~\eqref{eq:traj_loss_cont} over a
fixed continuation segment. Let $\nabla_\Theta\mathcal L$ denote the
branch-consistent gradient obtained by differentiating along the realized
smooth equilibrium branch, and let
$\nabla_\Theta\mathcal L_{\mathrm{proxy}}$ denote the gradient obtained by
freezing $\mathbf S(\lambda)$ on each continuation interval. Suppose
Assumption~\ref{assm:branch} holds, $\mathbf S(\lambda)$ has bounded
local variation on the segment, and $f$, its first derivatives, the
controller Jacobian $\partial\mathbf u_\Theta/\partial\Theta$, and the
objective derivatives are uniformly bounded along the realized rollout.
Then
\begin{equation}
\left\|
\nabla_\Theta\mathcal L_{\mathrm{proxy}}
-
\nabla_\Theta\mathcal L
\right\|
\le
C\,\eta_S,
\label{eq:app_grad_bound}
\end{equation}
where $C$ depends on the above bounds and the segment length, but not on
$\eta_S$. In particular, if $\mathbf S$ is Lipschitz continuous, then
$\eta_S=O(\Delta\lambda)$ and the proxy-gradient error is
$O(\Delta\lambda)$ on a fixed segment.
\end{proposition}

\begin{proof}
Under Assumption~\ref{assm:branch}, the implicit function theorem gives a
locally smooth equilibrium map and bounded local sensitivities. The exact
linearized branch dynamics uses the continuously varying tangent
$\mathbf S(\lambda)$, whereas the proxy backward pass replaces it on each
interval by the frozen value $\mathbf S_k$. The resulting perturbation in
the linearized dynamics is bounded by $\eta_k$ on interval $k$. Standard
stability estimates for linear variational and adjoint equations then
bound the corresponding costate error by a constant times $\eta_S$ over a
fixed segment. Multiplying by the bounded controller Jacobian and summing
over the segment yields Eq.~\eqref{eq:app_grad_bound}.
\end{proof}

The constant $C$ may grow with the segment length, which motivates the
receding-horizon design used in Neural Control.

\paragraph{Role of receding-horizon continuation.}
RHC reduces the accumulation of frozen-tangent error by applying the
proxy gradient over short segments and re-anchoring each segment at the
realized equilibrium. This does not remove the underlying
ill-conditioning near bifurcations or basin switches, but it limits the
horizon over which inaccurate local tangents can accumulate. When
$\mathbf G_x$ becomes singular, when the solver switches basins, or when
$\mathbf S(\lambda)$ varies rapidly, the bound above may no longer hold; these cases correspond to the primary failure modes of the frozen-tangent
approximation.

\section{Implementation Details and Hyperparameters}
\label{app:hyperparams}

This section reports the hyperparameters used in our numerical experiments (Secs.~\ref{subsec:main_results}--\ref{subsec:baselines_ablations}).
Unless otherwise stated, all methods share the same forward equilibrium simulator and the same controller parameterization; they differ only in the optimizer.

\paragraph{Elastic-strip simulator and equilibrium solver.}
We discretize the strip centerline into $N{+}1=101$ nodes and clamp the first two and last two nodes.
Unless otherwise stated, material parameters are
$k_b=7.85\times 10^{-8}~\mathrm{N\cdot m^2}$ and $k_s=0.314~\mathrm{N}$,
and we set $\mathbf g=\mathbf 0$ (no gravity) in all numerical experiments.

At each continuation step we solve for the free DoFs using a trust-region Newton--CG method (App.~\ref{app:strip_solver}),
warm-started from the previous equilibrium.
We use initial radius $\Delta_0=10^{-3}$, acceptance threshold $\eta=0.1$, shrink/expand factors $(0.5,2.0)$, and at most $200$ outer iterations.
We terminate when $\|\nabla E_\texttt{tot}\|\le 10^{-9}$ (gradient w.r.t.\ free DoFs) or when the radius drops below $10^{-8}$.

If the Hessian is detected SPD we take a Newton/dogleg step; otherwise we use Steihaug PCG in the trust region.
PCG is capped at $200$ iterations and stops when the PCG residual satisfies
$\|r\|\le \min(0.5,\sqrt{\|\nabla E\|})\|\nabla E\|$, on negative curvature, or upon hitting the trust-region boundary.
We use a diagonal (Jacobi) preconditioner.

\paragraph{Controller parameterization and adjoint optimization.}
We parameterize the continuation controller $u_\Theta(\lambda)$ as a lightweight MLP that takes the scalar $\lambda\in[0,1]$ as input and outputs the control vector $u(\lambda)\in\mathbb R^{d_u}$.
For Tasks~1--2, $d_u=2$ (left/right transverse displacement rates); for Task~3, $d_u=3$ (end-effector $(x,y)$ displacement rates and in-plane rotation rate).
Unless otherwise stated, we use $2$--$3$ hidden layers with width $64$ (architectures $[64,64]$ or $[64,64,64]$), ReLU activations, and a final $\tanh$ to bound the output; we scale the output to enforce $\|u(\lambda)\|\le u_{\max}$.

For the adjoint-only baseline, we use a single rollout with full horizon $K=101$ continuation steps.
For Adjoint+RHC, we use segment horizon $H=10$ continuation steps per segment (with horizon shifting as described in Sec.~\ref{sec:receding_horizon_control}).
Both adjoint-based methods use Adam with learning rate $10^{-2}$ to optimize $\Theta$.

\paragraph{SPSA baseline.}
SPSA is used as a black-box optimizer over the same controller parameters $\Theta$.
We use a two-sided perturbation estimator with perturbation magnitude $c=5\times 10^{-3}$ and learning rate $10^{-2}$.
Each SPSA update uses one perturbation pair, corresponding to two objective evaluations.
We choose a relatively large $c$ to obtain a stable finite-difference directional signal under strongly nonlinear deformation; with unit-normalized strip length, $c=5\times10^{-3}$ corresponds to a small ($\sim 0.5\%$) geometric perturbation that empirically yields more reliable search directions than smaller perturbations that can be dominated by solver tolerance and numerical noise.

\paragraph{iCEM baseline.}
iCEM is used as a black-box optimizer over the same controller parameters $\Theta$.
At each update, iCEM samples a population of candidate controller parameters, evaluates each candidate by a full equilibrium-defined rollout, selects elite candidates, smooths the population update, and reinserts the current best candidate.
We use population size $P=10$, elite fraction $0.3$, smoothing coefficient $\alpha=0.25$, minimum standard deviation $10^{-2}$, elite reuse fraction $1.0$, no population decay, and best-candidate reinsertion.
When the initial standard deviation is not specified, it is set to $0.5u_{\max}$.
We use temporally correlated perturbations with correlation parameter $\beta=2.0$ in the controller-induced actuation schedule.

\paragraph{Update budgets and reporting.}
Task~1 uses $200$ optimizer updates, while Tasks~2--3 use $500$ optimizer updates.
For RHC, an update refers to one optimizer update within the current segment, and reported task losses are computed on the concatenated executed trajectory across segments.
Wall-clock time in Table~\ref{tab:theory_empirical} includes all equilibrium solves and the additional tangent/adjoint linear solves used by implicit sensitivities.
CEM hyperparameters and results are reported separately in Appendix~\ref{app:cem_results}.

\section{Additional CEM Baseline Results}
\label{app:cem_results}

In the main text, we report SPSA and iCEM as the primary gradient-free
baselines. Here we additionally report results for the classical
cross-entropy method (CEM), which was used in the original submission.
CEM is applied as a black-box optimizer over the same controller
parameters $\Theta$ and is run without RHC. At each update, CEM samples a
population of candidate controller parameters, evaluates each candidate
using a full equilibrium-defined rollout, selects elite candidates, and
updates the search distribution. We use population size $P=10$ and elite
fraction $\rho=0.3$.

\begin{table*}[h]
\centering
\scriptsize
\setlength{\tabcolsep}{5pt}
\resizebox{\textwidth}{!}{%
\begin{tabular}{lcc|cc|cc}
\hline
& \multicolumn{2}{c|}{Task 1 (200 updates)}
& \multicolumn{2}{c|}{Task 2 (500 updates)}
& \multicolumn{2}{c}{Task 3 (500 updates)} \\
Method
& Time (s)$\downarrow$
& Best loss$\downarrow$
& Time (s)$\downarrow$
& Best loss$\downarrow$
& Time (s)$\downarrow$
& Best loss$\downarrow$ \\
\hline
CEM
& $2043.2\pm73.0$
& $6.7{\times}10^{-4}\pm1.2{\times}10^{-3}$
& $5358.9\pm208.7$
& $6.4{\times}10^{-3}\pm3.9{\times}10^{-3}$
& $3176.4\pm232.8$
& $1.6{\times}10^{-2}\pm8.5{\times}10^{-3}$ \\

\textbf{Adjoint+RHC}
& $\mathbf{16.1\pm2.8}$
& $\mathbf{2.3{\times}10^{-7}\pm3.0{\times}10^{-7}}$
& $\mathbf{186.9\pm24.8}$
& $\mathbf{3.6{\times}10^{-8}\pm3.6{\times}10^{-8}}$
& $\mathbf{50.9\pm6.1}$
& $\mathbf{3.8{\times}10^{-8}\pm7.4{\times}10^{-9}}$ \\
\hline
\end{tabular}%
}
\caption{\textbf{Additional CEM baseline.}
CEM is evaluated under the same fixed update budgets as the main
experiments and uses the same controller parameterization
$\mathbf u_\Theta(\lambda)$. Each CEM candidate requires a full
equilibrium-defined rollout. The results are consistent with the main
comparison: Adjoint+RHC achieves substantially lower loss with much less
wall-clock time.}
\label{tab:cem_results}
\end{table*}

\section{Learned DEQ-Style Implicit Equilibrium Model}
\label{app:deq_validation}

To test whether Neural Control applies beyond analytical mechanics
models, we evaluate it on a learned DEQ-style implicit equilibrium model
trained from experimental force--strain data. We collect force--strain
measurements from a slinky and train a neural energy model whose
equilibrium response is defined implicitly. After training, the model is
frozen and used as the forward implicit model for Neural Control. We then
optimize a force trajectory so that the equilibrium strain tracks a
prescribed sinusoidal reference. Local sensitivities are computed by
implicit differentiation of the learned equilibrium residual, using the
same proxy-adjoint pipeline as in the elastic-strip experiments.

The learned implicit model closely fits the experimental force--strain
response, and Neural Control successfully optimizes the input force
trajectory to track the target strain. The resulting tracking losses over
the optimized segments are on the order of $10^{-8}$--$10^{-11}$. This
validation demonstrates that the proposed proxy-adjoint continuation
framework is not tied to analytical elastic-strip simulators; it can also
operate on learned implicit equilibrium models, provided that local
implicit sensitivities are available. Additional plots, including the
force--strain fit, training curve, and tracking result, are provided on
the project page\textsuperscript{\ref{fn:projectpage}}.

\section{System Identification for the Elastic Strip in Real Robotic Manipulation}
\label{app:sysid}

This section describes the lightweight system-identification procedure used to match the elastic-strip model to the real-robot experiments.
The key identified quantity is the gravito-bending length $L_{gb}$, which captures the characteristic deformation scale of a slender strip under gravity and is robust to many nuisance factors, such as camera noise and slight boundary mismatch.

\paragraph{Gravito-bending length.}
For an inextensible/slender strip under its own weight, the dominant competition is between bending stiffness and distributed gravity load. We parameterize this balance by
\begin{equation}
L_{gb}
\;:=\;
\left(\frac{k_b}{\rho w h g}\right)^{\tfrac{1}{3}},
\label{eq:lgb_def}
\end{equation}
where $k_b$ is the bending stiffness (units N$\cdot$m$^2$), $\rho$ is the density, $w$ is the strip width, $h$ is the thickness, and $g$ is gravitational acceleration.

We collect $M=5$ static snapshots of the strip under gravity with diverse boundary conditions (multiple clamp-clamp configurations).
Each snapshot provides an observed equilibrium shape $\hat{x}^{(m)}$ together with the corresponding boundary constraints.
Using multiple configurations reduces degeneracy and discourages overfitting to a single pose.

For a candidate $L_{gb}$, we run the simulator under each experimental boundary condition and obtain the converged equilibrium shape
$x^{(m)}(L_{gb})$.
We measure mismatch on \emph{free nodes only} (excluding constrained/clamped DoFs) and use a mean-squared shape discrepancy:
\begin{equation}
J(L_{gb})
\;=\;
\frac{1}{M}\sum_{m=1}^{M}
\frac{1}{|\mathcal I_f^{(m)}|}
\sum_{i\in \mathcal I_f^{(m)}}
\left\|
x^{(m)}_i(L_{gb})-\hat{x}^{(m)}_i
\right\|_2^2,
\label{eq:lgb_obj}
\end{equation}
where $\mathcal I_f^{(m)}$ denotes the free-node index set for snapshot $m$.
We identify
\begin{equation}
L_{gb}^\star \in \arg\min_{L_{gb}\in [L_{\min},L_{\max}]} J(L_{gb}).
\label{eq:lgb_opt}
\end{equation}

We solve Eq.~\eqref{eq:lgb_opt} with a bounded one-dimensional scalar minimization (Brent-type method).
Unless otherwise stated, we use bounds $[L_{\min},L_{\max}]=[0.05,0.10]~\mathrm{m}$ and tolerance \texttt{xatol}$=10^{-4}$.
Each objective evaluation runs the forward simulator with gravity enabled and otherwise matches the real-robot settings.
The identified value is $L_{gb}=0.0514~\mathrm{m}$.

\end{document}